\documentclass[11pt]{article}

\usepackage{iftex}

\usepackage{fullpage}

\usepackage{amssymb,amstext,xspace}
\usepackage{algorithmic}
\usepackage{algorithm}

\newcommand{\E}{\mathbb{E}}

\newcommand{\myepsilon}{\epsilon_{0}}

\newcommand{\sm}{\nu}

\usepackage{amsmath,amsthm, amssymb}
\usepackage{graphicx}
\usepackage{mfirstuc}
\usepackage{hyperref}
\hypersetup{colorlinks=true,linkcolor=blue,citecolor=blue,urlcolor=blue}

\newtheorem*{theorem*}{Theorem}
\newtheorem*{definition*}{Definition}

\newtheorem{theorem}{Theorem}
\newtheorem{lemma}{Lemma}

\newtheorem{claim}{Claim}

\usepackage{thm-restate}

\usepackage[utf8]{inputenc} % allow utf-8 input
\usepackage[T1]{fontenc}    % use 8-bit T1 fonts
\usepackage{hyperref}       % hyperlinks
\usepackage{url}            % simple URL typesetting
\usepackage{booktabs}       % professional-quality tables
\usepackage{amsfonts}       % blackboard math symbols
\usepackage{nicefrac}       % compact symbols for 1/2, etc.
\usepackage{microtype}      % microtypography
\usepackage[numbers]{natbib}

\usepackage{times}
\usepackage{hyperref}
\hypersetup{colorlinks=true,linkcolor=blue,citecolor=blue,urlcolor=blue}

\begin{document}

\title{An Optimal Elimination Algorithm for \\ Learning a Best Arm}

%\author{Avinatan Hassidim, Ron Kupfer, Yaron Singer}
%\author{%
%	\Name{Avinatan Hassidim} \mail{avinatan@cs.biu.ac.il}\\
%	\addr   Bar Ilan University and Google
%	\AND
%	\Name{Ron Kupfer} \mail{ron.kupfer@mail.huji.ac.il}\\
%	\addr  The Hebrew University %
%	\AND
%	\Name{Yaron Singer} \mail{yaron@seas.harvard.edu}\\
%	\addr  Harvard University %
%}

\author{%
	Avinatan Hassidim\\	Bar Ilan University and Google \\   \texttt{avinatan@cs.biu.ac.il} \\
	\and Ron Kupfer \\	 The Hebrew University \\      \texttt{ron.kupfer@mail.huji.ac.il}\\
     \and Yaron Singer \\	 Harvard University \\	 \texttt{yaron@seas.harvard.edu} 
}
\date{}
\maketitle

%\begin{keywords}%
%bandits, best arm
%\end{keywords}

\begin{abstract}
We consider the classic problem of $(\epsilon,\delta)$-\texttt{PAC} learning a best arm where the goal is to identify with confidence $1-\delta$ an arm whose mean is an $\epsilon$-approximation to that of the highest mean arm in a multi-armed bandit setting.  
This problem is one of the most fundamental problems in statistics and learning theory, yet somewhat surprisingly its worst case sample complexity is not well  understood.  In this paper we propose a new approach for $(\epsilon,\delta)$-\texttt{PAC} learning a best arm.  This approach leads to an algorithm whose sample complexity converges to \emph{exactly} the optimal sample complexity of $(\epsilon,\delta)$-learning the mean of $n$ arms separately and we complement this result with a conditional matching lower bound.  More specifically:

\begin{itemize}
\item The algorithm's sample complexity converges to \emph{exactly} $\frac{n}{2\epsilon^2}\log \frac{1}{\delta}$ as $n$ grows and $\delta \geq \frac{1}{n}$;  
\item We prove that no elimination algorithm obtains sample complexity arbitrarily lower than $\frac{n}{2\epsilon^2}\log \frac{1}{\delta}$.  Elimination algorithms is a broad class of $(\epsilon,\delta)$-\texttt{PAC} best arm learning algorithms that includes many algorithms in the literature.   
\end{itemize}
When $n$ is independent of $\delta$ our approach yields an algorithm whose sample complexity converges to $\frac{2n}{\epsilon^2} \log \frac{1}{\delta}$ as $n$ grows.  In comparison with the best known algorithm for this problem our approach improves the sample complexity by a factor of over 1500 and over 6000 when $\delta\geq \frac{1}{n}$.
\end{abstract}

\section{INTRODUCTION}
In this paper we study the classic problem of $(\epsilon,\delta)-\texttt{PAC}$ learning a best arm. 
In this problem there is a set $A$ of $n$ arms and sampling an arm $a\in A$ generates a random variable $\xi(a)$ drawn from some unknown distribution ${\mathcal{D}({a}) \subseteq [0,1]}$\footnote{All the results in this paper can be generalized for any sub-Gaussian distribution 
%by scaling down the values and adjusting the selection of $\epsilon$ 
as discuss in Appendix \ref{sec:gauss}.}. The mean of every arm $a$ is denoted $\mu(a)$ and an \emph{optimal} arm is $a^\star \in \arg\max_{a\in A}\mu(a)$.  A strategy $(\epsilon,\delta)$-learns the best arm if it returns $a\in A$ s.t. $\mu(a) \geq \mu(a^\star) -\epsilon$ with confidence at least $1-\delta$ over the arm distribution and randomization of the strategy.  The goal is to $(\epsilon,\delta)$-learn the best arm with minimal worst case sample complexity over all distributions in $[0,1]$.  

By the celebrated Hoeffding bound we know that it suffices to sample each arm $\frac{1}{2\epsilon^2}\log \frac{1}{\delta}$ times to ensure we are $\epsilon$-close to its true mean with confidence $1-\delta$, and that without additional information this bound is optimal.  A trivial solution is then to estimate the mean of each arm using sufficiently-many samples and take the arm whose empirical mean is largest.  A trivial upper bound for learning a best arm using this approach is $\frac{2n}{\epsilon^2}\log\frac{n}{\delta}$.

In a seminal paper, Even-Dar et al. considered the problem of $(\epsilon,\delta)$-learning a best arm when the number of arms $n$ is asymptotically large~\cite{EMM03}.  They introduce \textsc{Median Elimination} which is an ${(\epsilon,\delta)}$-learning strategy whose sample complexity is $\mathcal{O}\left (\frac{n}{\epsilon^2}\log \frac{1}{\delta} \right)$.  To date, \textsc{Median Elimination} is the best algorithm for provably $(\epsilon,\delta)$-learning a best arm in terms of sample complexity when $n$ is sufficiently large.  As such it is a fundamental building block in a variety of algorithms (see e.g. ~\cite{KS10,KKS13,UCFN13,JN14,CLTL15}), and has applications in a broad range of domains.
Unfortunately, the constant terms hiding in the $\mathcal{O}$ notation of the sample complexity of \textsc{Median Elimination} are quite large.  For $n=100$ its sample complexity exceeds ${1000\times (\frac{n}{\epsilon^2} \log \frac{1}{\delta})}$, and grows to over 3 times as $n$ grows.  

In terms of lower bounds, the best known bound for this problem is by Manor and Tsitisklis who show that $\frac{n}{128\epsilon^2}\log \frac{1}{4\delta}$ samples are necessary for $(\epsilon,\delta)$-learning a best arm~\cite{MT03}. Thus, the gap between the best known upper and lower bounds exceeds 300,000 and begs the obvious question:
\begin{center}
\emph{What is the optimal sample complexity of PAC learning a best arm?}
\end{center}

\paragraph{Main contribution.}
In this paper we address this question and take fundamentally new approaches to obtain upper and lower bounds for $(\epsilon,\delta)$-learning a best arm.  At a high level, 
our algorithms are designed so that their probability of failure diminishes as the number of arms grows.  For a lower bound, we observe that our algorithm as well as many other algorithms for learning a best arm in the literature can be broadly characterized as iteratively sampling and discarding arms until one arm is left. We call algorithms that fit this description \emph{elimination algorithms} and prove a tight lower bound on this class that matches our upper bound.  Our results can be summarized as follows:

%in the following manner: in the first iteration all arms are considered; at every iteration every arm still considered is pulled and a subset of arms is then discarded and ceases to be considered; in the final iteration one arm remains and selected.  
 %An algorithm is an \emph{elimination algorithm} if it   This class of algorithms includes every algorithm for $(\epsilon,\delta)$-\texttt{PAC} learning a best arm that we are aware of.  

%
\begin{enumerate}
\item We describe a new algorithm whose sample complexity converges with $n$ to \emph{exactly} $\frac{n}{2\epsilon^2}\log \frac{1}{\delta}$ when $n\geq \frac{1}{\delta}$.  This bound \emph{exactly} matches the sample complexity of $(\epsilon,\delta)$-\texttt{PAC} learning the mean of each arm separately according the Hoeffding bound. In comparison to \textsc{Median Elimination} the sample complexity is lower by a factor greater than $6000$ when $n$ is large;
\item When $n$ is independent of $\delta$, we describe a simplified version of the algorithm whose sample complexity converges to $\frac{2n}{\epsilon^2} \log \frac{1}{\delta}$;  Furthermore, for any $\delta< 0.05$, any $n>0$ and $\epsilon \in (0,1)$ our approach yields an algorithm whose sample complexity is $\frac{18n}{\epsilon^2}\log\frac{1}{\delta}$.  In comparison to \textsc{Median Elimination} this reduces the sample complexity by a factor greater than $300$;
\item We prove that the number of samples any elimination algorithm requires to $(\epsilon,\delta)$-learn a best arm is arbitrarily close to $\frac{n}{2\epsilon^2}\log \frac{1}{\delta}$.  
\end{enumerate}

Our results are in the standard $(\epsilon,\delta)$-\texttt{PAC} learning model, i.e. the goal is to find an $\epsilon$-best arm with probability $1-\delta$ and sample complexity is measured in the worst case across any distribution in $[0,1]$ (or any subgaussian, see Appendix~\ref{sec:gauss}).
%Before moving forward, it would be instructive to discuss this problem setting as well as closely related settings. %Additional related work is in Appendix~\ref{sec:related}. 
\subsection{Related work}

The study of learning the best arm dates back to classic work by~\cite{C59}, and later by ~\cite{A61}, ~\cite{KS63}, and ~\cite{K84}.  More recently, $(\epsilon,\delta)$-\texttt{PAC} guarantees were studied in~\cite{DGW02} and later by~\cite{EMM03,MT03}.  There have since been other variants of this problem studied, including \texttt{PAC} learning a set of arms~\cite{BMS09,TAS12,KK13,BWV13}, or the fixed budget setting where the goal is to minimize $\delta$ subject to a budget constraint on samples~\cite{BMS09,ABM10,GGL12}.  

\paragraph{Learning an $\epsilon$-best arm.}As the state-of-the-art algorithm for $(\epsilon,\delta)$-PAC learning a best arm, \textsc{Median Elimination} is widely used as a sub-procedure (e.g. ~\cite{KS10,KKS13,UCFN13,JN14,CLTL15,STK16}). An improvement on its sample complexity as suggested here achieves dramatically lower sample complexity for all procedures that employ \textsc{Median Elimination}. The interesting regime in this problem setting is the one where $n$ is large, as otherwise it suffices to use the naive sampling strategy of sampling each arm with approximation $\frac{\epsilon}{2}$ and confidence $\frac{\delta}{n}$ and selecting the arm with largest empirical mean.\footnote{In particular, our algorithms use the naive elimination strategy when $n<10^5$. For \textsc{Median Elimination} the naive strategy has better sample complexity for any $n<2^{1500}$.}

\paragraph{Elimination Algorithms.}
A common approach for the $(\epsilon,\delta)$-\texttt{PAC} problem, is using algorithms who are based on elimination process such as the Median Elimination by~\cite{DGW02} and~\cite{EMM03}. 
In this framework, the algorithm may be described as series of rounds, where at each round we sample all non-eliminated arms and at the end of each round we may eliminate some of the arms until reaching a conclusion.
Our work focuses on this family of algorithm and we show a lower bound for those algorithms that match our upper bound.
Our lower bound hold for this class of algorithms.

\paragraph{Lower bounds.}
\cite{MT03} show that $\frac{n}{128\epsilon^2}\log \frac{1}{4\delta}$ samples are necessary for $(\epsilon,\delta)$-learning a best arm.
As mention before,~\cite{degenne2019pure} show that their algorithm which is based on track-and-stop is tight instance-wise for arm distributions that comes from one-parameter one-dimensional canonical exponential families. The lower bound hold for any fixed number of arms as $\delta$ goes to $0$.
This lower bound is instance specific and it not clear on how to deduce worst case lower bound for all instances.
Recently,~\cite{katz2019true} showed that $\Theta\left(\frac{n}{m}\right)$ samples are needed and sufficient when $n$ is the number of arms, $m$ is the number of $\epsilon$-best arms, and $\delta,\epsilon$ are constants.

%\subsection{Related work}
%The study of learning the best arm dates back to classic work by~\cite{C59}, and later by ~\cite{A61}, ~\cite{KS63}, and ~\cite{K84}.  More recently, $(\epsilon,\delta)$-\texttt{PAC} guarantees were studied in~\cite{DGW02} and later by~\cite{EMM03,MT03}.  There have since been other variants of this problem studied, including \texttt{PAC} learning a set of arms~\cite{BMS09,TAS12,KK13,BWV13}, or the fixed budget setting where the goal is to minimize $\delta$ subject to a budget constraint on samples~\cite{BMS09,ABM10,GGL12}.  

\paragraph{Learning an exact best arm.} In the exact best arm learning problem the goal is to $(0,\delta)$-\texttt{PAC} learn the best arm (see e.g.~\cite{ABM10,KKS13,JN14,JMNB14,SJR16,R16,GK16}).
This problem is computationally more demanding as arm means can be arbitrarily close and one seeks optimal sample complexity that depends on the arm distributions.  For exact best arm learning several algorithms use $\epsilon$-best arm learning as a subroutine, where our work is directly applicable (e.g. ~\cite{KKS13,JN14,JMNB14}).  For exact best arm learning, the optimal sample complexity bounds for exponential distributions is achieved in~\cite{GK16}.  %Despite the obvious relationship between $(\epsilon,\delta)$-\texttt{PAC} and $(0,\delta)$-\texttt{PAC} learning the best arm, we are not aware of a reduction between these problems.

\paragraph{Instance based analysis.} 
The nature of exact best arm learning necessitates specific assumptions about the relevant families of distributions for the arms. 
This motivates a series of works that deviate from the $(\epsilon,\delta)$-\texttt{PAC} learning setting where the sample complexity is worst case across all distributions. In particular, a recent line of work analyzes the sample complexity as a function of the given instance (i.e. set of distributions) and $\delta$ for both exact best arm and $\epsilon$-best arm problems~\cite{GK16,degenne2019non,garivier2019non,degenne2019pure}. 
In this genre, variants of \textit{explore and exploit} algorithms known as \textit{track-and-stop} algorithms turned out to be efficient in the number of samples under some assumptions.
For for $\epsilon$-best arm, an instance-based optimal algorithm was shown in~\cite{garivier2019non} under the assumption that there is such a unique arm. Recently,~\cite{degenne2019pure} show how to generalize this approach without assuming a unique $\epsilon$-best arm. By using a function $T(\bar{\mu})$ from set of distributions to the reals, they show that for any instance $\bar{\mu}$ which belongs to the one-parameter one-dimensional canonical exponential family, $(1+o(1))T(\bar{\mu})\log \frac{1}{\delta}$ samples are necessary and sufficient for $(\epsilon,\delta)$-learning a best arm, when $n$ is fixed and $\delta$ goes to $0$.

\paragraph{From instance-based to worst case analysis.} When the number of arms $n$ is fixed and $\delta$ goes to $0$ and the distribution is bounded in $[0,1]$, a worst case sample complexity bound can be trivially achieved via the naive elimination strategy. Thus, while this is an interesting regime for instance-based analysis, it is not interesting for worst case analysis. On the other hand, when fixing $\delta$ and letting the number of arms grow, it is not clear what is the asymptotic sample complexity of the problem in worst case, and it cannot be deduced from the instance based analysis. The main contribution of our work is showing upper and lower bounds for this problem.   

\paragraph{Running time.} Beyond worst case vs. instance based guarantees, elimination algorithms are exponentially faster compared to other approaches like track-and-stop. The algorithms we present here run in $\mathcal{O}(\log^2 n)$ parallel time in the PRAM model \cite{hagerup1989optimal}, hence giving a total implementation in poly-logarithmic time complexity which is an exponential improvement compared to~\cite{garivier2019non,degenne2019pure}.

\paragraph{Implications}
Obtaining algorithms with dramatic lower sample complexity for a basic problem like learning a best arm can have several consequences.  First, all previous algorithms that seek provable guarantees and directly employ \textsc{Median Elimination}  (e.g. ~\cite{KS10,KKS13,UCFN13,JN14,CLTL15,STK16}) can use the algorithms here instead and achieve dramatically lower sample complexity.  %From a purely theoretical perspective, $\frac{n}{2\epsilon^2}\log\frac{1}{\delta}$ is a natural bound as it corresponds to the number of samples needed to $(\epsilon,\delta)$-learn the mean of every arm separately.  The matching lower bound on \emph{elimination algorithms} stre
%%It also seems reasonable to conjecture that $\frac{n}{\epsilon^2}\log\frac{1}{\delta}$ is the optimal sample complexity for learning a best arm. 
From a practical perspective, \textsc{Median Elimination} is not a particularly good choice.  The naive sampling strategy of sampling each arm with approximation $\frac{\epsilon}{2}$ and confidence $\frac{\delta}{n}$ and selecting the arm with largest empirical mean $(\epsilon,\delta)$-learns a best arm and has lower sample complexity than \textsc{Median Elimination} whenever the number of arms is smaller than $2^{1500}$.  Nevertheless there is a great deal of work on heuristics based on \textsc{Median Elimination}.  Our hope is that some of the ideas presented here would not only contribute to \emph{provably} learning a best arm, but also heuristics.

\subsection{Paper organization}
We present our algorithms in order of increasing complexity.  The first is the \textsc{Simple Approximate Best Arm} algorithm introduced in Section~\ref{sec:saba} which makes assumptions about the input.  In Section~\ref{sec:aba} we present \textsc{Approximate Best Arm} which removes these assumptions and achieves sample complexity $ \frac{18 n }{\epsilon^2} \log \frac{1}{\delta}$ for $\delta<0.05$ and any $n$ which easily generalizes to achieve a bound that converges to $\frac{2 n }{\epsilon^2} \log \frac{1}{\delta}$ as $n$ grows.  In Section~\ref{sec:abaleh} we present the \textsc{Approximate Best Arm Likelihood Estimation by Hoeffding} whose sample complexity asymptotically matches the Hoeffding bound of estimating the mean of every arm separately. Lastly, our lower bound is presented in Section~\ref{sec:lower}. In Appendix \ref{sec:Experiments} we show simulations demonstrating that in practice, there is a large gap between the sample complexity of our algorithms and  \textsc{Median Elimination}.

%\input{related}
%\input{prelim}
%\newpage

\section{SIMPLE APPROXIMATE BEST ARM ALGORITHM}\label{sec:saba}
In this section we present the Simple Approximate Best Arm (\textsc{Saba}) algorithm. \textsc{Saba} is a simplified version of the algorithm described in the next section.  Its simplicity is achieved by making assumptions about the input to provably $(\epsilon,\delta)$-learn an a best arm. Namely, it assumes that $n\geq \max\{10^5,1/\delta^4\}$ and that there is a unique $\epsilon$-best arm, i.e. all the arms in the input are $\epsilon$-far from $a^\star$.  \textsc{Saba} is a concatenation of two procedures.  The first is \textsc{Aggressive Elimination} which is the main algorithmic idea behind this paper.  The second is \textsc{Na\"ive Elimination} which trivially samples all arms sufficiently many times and selecting the one with largest empirical mean.

\subsection{Na\"ive Elimination}
The following procedure is the na\"ive sampling approach to finding a best arm.  

\begin{algorithm}[H]
\caption{\textsc{Na\"ive Elimination}}
\label{algo:NAIVE}
\begin{algorithmic}[1]
    \INPUT $\epsilon,\delta>0$, arms $A$, noisy oracle for $\mu:A \to [0,1]$	
    \OUTPUT arm in $A$ with largest empirical mean with $\frac{2}{\epsilon^2}\log \frac{|A|}{\delta}$ samples	\end{algorithmic}
\end{algorithm}
The sample complexity of \textsc{Na\"ive Elimination} is trivially $\frac{2|A|}{\epsilon^2}\log\frac{|A|}{\delta}$ and it returns an arm that is $\epsilon$-close to $a^\star$ with probability at least $1-\delta$.  We say that an arm $a \in A$ is \textbf{$\eta$-close} to $a' \in A$ if $\mu(a') - \mu(a) \leq \eta$ and \textbf{$\eta$-far} if $\mu(a') - \mu(a) >\eta$. One can obtain the approximation and confidence by bounding the likelihood of underestimating $a^\star$ and overestimating arms that are $\epsilon$-far from $a^\star$.  For completeness we give full details in Appendix~\ref{sec:saba_appendix}.  Throughout the paper we repeatedly use \textsc{Na\"ive Elimination} with different values of $n$ and various approximation and confidence parameters.  

%The sample complexity of \textsc{Na\"ive Elimination} is trivially $\frac{2|A|}{\epsilon^2}\log\frac{|A|}{\delta}$ and it returns an arm that is $\epsilon$-close to $a^\star$ with probability at least $1-\delta$.  We say that an arm $a \in A$ is \textbf{$\eta$-underestimated} if its empirical mean $\hat{\mu}(a)$ is evaluated to be less than $\eta$ of its true value, i.e. $\hat{\mu}(a)< \mu(a)- \eta$. An arm $a \in A$ is \textbf{$\eta$-overestimated} if $\hat{\mu}(a)>\mu(a)+\eta$.  For completeness we give full details in Appendix~\ref{sec:saba_appendix}.  Throughout the paper we repeatedly use \textsc{Na\"ive Elimination} with different values of $n$ and various approximation and confidence parameters.  

\subsection{Agressive Elimination}
The \textsc{Aggressive Elimination} procedure that we introduce here iteratively discards arms with low empirical mean until reducing the total number of arms to $\frac{n^{3/4}}{2}$.  To do so, in each round $i$ the procedure samples every arm $(i+1)\frac{2}{\epsilon^2}\log\frac{1}{\delta}$ times and selects the $(\delta+\phi(n))$ fraction of arms whose sampled mean is highest into the next round.  Intuitively, $\phi(n)$ is a small fraction s.t. the $(\delta+\phi(n))$ fraction of arms with largest sampled mean is likely to include $a^\star$.  It is technically defined as:
\begin{equation}\label{eq:phi_dfn}
\phi(n) = \sqrt{ \frac{6\log(n)}{n^{3/4}} }.%\sqrt{ \frac{4\left (  \log(10) + 2\log(n) + \log\log (n)\right )}{n^{3/4}} }
\end{equation} 
We will rely on this definition in Lemma~\ref{lem:overestimate} when analyzing the likelihood of $a^\star$ remaining in the final set of arms returned by the procedure.  In particular, we bound the likelihood that $a^\star$ is underestimated and that other arms are overestimated.  This definition of $\phi(n)$ is designed in such a way that we can later bound the likelihood that too many arms are overestimated, under certain assumptions.     

The second term we define is $t(n)$ which is the number of iterations \textsc{Aggressive Elimination} requires until reaching $\frac{n^{3/4}}{2}$ arms when we shrink the number of arms in each iteration by $\delta+{\phi(n)}$:
\begin{equation}\label{eq:t_n}
t(n) = \left\lceil \frac{\log n + 4\log 2}{4\log \left ( \frac{1}{\delta+\phi(n)} \right)}\right\rceil.
\end{equation}
Given these definitions we now formally describe and analyze \textsc{Aggressive Elimination} below.
\begin{algorithm}[H]
\caption{\textsc{Aggressive Elimination}}
\label{algo:AE}
\begin{algorithmic}[1]
    \INPUT $\epsilon,\delta>0$, arms $A_0$, noisy oracle for $\mu:A_0 \to [0,1]$
	\FOR{$i\in \{0,1,2,\ldots, t(n) \}$}{
	\STATE apply $\ell_{i\texttt{+}1} =  (i+1)   \left\lceil \frac{2}{\epsilon^2} \log \frac{1}{\delta}  \right\rceil$ samples $\forall a \in A_i$
	\STATE $A_{i\text{+}1} \leftarrow$  the $|A_i|\times\left\lfloor\delta + \phi(n)\right\rfloor$ best arms in $A_{i}$}
	\ENDFOR \\ %\hspace{0.1in}
		\OUTPUT $A_{t(n)+1}$
	\end{algorithmic}
\end{algorithm}

\paragraph{Sample complexity.}  We will express the sample complexity of \textsc{Aggressive Elimination} using $G(n,\delta)$ defined below.  Importantly, $G(n,\delta)$ converges to $0$ as $n$ grows and $\delta$ goes to $0$:
\begin{equation}
G(n,\delta) = \sum_{i=1}^{t(n)} (\delta+ \phi(n))^{i} (i+1)
\end{equation}
\begin{claim}\label{clm:aggressive}
$\forall \epsilon,\delta\in[0,1]$, $n\geq 1$ the sample complexity of \textsc{Aggressive Elimination} is bounded by:
$$ \Big(1+ G(n,\delta) \Big ) \times \left\lceil \frac{2n}{\epsilon^2} \log \frac{1}{\delta} \right\rceil .$$
\end{claim}
\begin{proof}
Each iteration $i$ uses $\ell_{i\text{+}1}=  (i+1)  \left\lceil\frac{2}{\epsilon^2}  \log \frac{1}{\delta}\right\rceil$ estimates on $|A_i| \leq  n(\delta + \phi(n))^i$ arms.  In total:  %In total, the number of samples is:
$$
%\begin{align*}
\sum_{i=0}^{t(n)} |A_i| \times  \ell_{i\text{+}1}
\leq   \sum_{i=0}^{t(n)} n \left (\delta + \phi(n) \right )^{i}  ( i+1) \times  \left\lceil\frac{2}{\epsilon^2}\log \frac{1}{\delta}\right\rceil 
%&=  \frac{2n} {\epsilon^2} \log  \frac{1}{\delta} \left ( 1 + \sum_{i=1}^{t(n)} (\delta+ \phi(n))^{i} (i+1) \right )\\
=  \Big (1+ G(n, \delta)\Big) \times \left\lceil\frac{2n} {\epsilon^2} \log  \frac{1}{\delta}\right\rceil.
%\end{align*}
\qedhere
$$
\end{proof}
Later in the paper we ignore the rounding of $\lceil\frac{2} {\epsilon^2} \log  \frac{1}{\delta}\rceil$ and $\left\lfloor(\delta + \phi(n))\right\rfloor$ when clear that the effect is negligible.
The important takeaway is that the sample complexity of \textsc{Aggressive Elimination} converges to $\frac{2n}{\epsilon^2}\log \frac{1}{\epsilon}$ as the number of arms grows and $\delta$ becomes small because $\lim_{n\to \infty, \delta\to 0}G(n,\delta)=0$. Later in the paper we usually use non-asymptotic notion of $\delta$, and $G(n,\delta)$ is estimated more carefully.

\paragraph{Likelihood of $a^\star$ surviving.}  Next we analyze the likelihood of the best arm $a^\star$ to appear in the $\frac{n^{3/4}}{2}$ arms output of the \textsc{Aggressive Elimination} procedure.  We begin with a simple lemma that analyzes the likelihood of $|A_i|\cdot (\delta+ \phi(n))$ arms -- the number of arms with largest empirical mean we select in each iteration  -- to be $\frac{\epsilon}{2}$-overestimated.  An arm $a \in A$ is \textbf{$\eta$-underestimated} if its empirical mean $\hat{\mu}(a)$ is evaluated to be less than $\eta$ of its true value, i.e. $\hat{\mu}(a)< \mu(a)- \eta$. An arm $a \in A$ is \textbf{$\eta$-overestimated} if $\hat{\mu}(a)>\mu(a)+\eta$.  The proof is deferred to Appendix~\ref{lem:overestimate_app}.  

%We show that as long as there are at most $\frac{n^{3/4}}{4}$ arms that are $\epsilon$-close to $a^\star$, then in every iteration $i$ the best arm $a^\star$ survives the elimination process and is included into $A_{i+1}$ with sufficiently high probability.  

\begin{restatable}{rLem}{overestimate}\label{lem:overestimate}
For every iteration $i\in\{0,1,\ldots,t(n)\}$ of \textsc{Aggressive Elimination} the probability that more than $|A_i|\cdot (\delta+ \phi(n))$ arms are $\frac{\epsilon}{2}$-overestimated at iteration $i$ is smaller than $\frac{1}{n^6}$.
\end{restatable}

The main idea that we now show is that with sufficient probability in every round, $a^\star$ is not $\frac{\epsilon}{2}$-underestimated and sufficiently few $\epsilon$-far arms are $\frac{\epsilon}{2}$-overestimated.  Showing this implies that in every round $a^\star$ is one of the arms with highest empirical mean and selected to the next round.

\begin{claim}\label{clm:unique}
Suppose the $\epsilon$-best arm $a^\star$ is unique, i.e. all arms are $\epsilon$-far from $a^\star$.  Then,  the likelihood that \textsc{Aggressive Elimination} returns a set of arms $A_{t(n)+1}$ that does not contain $a^\star$ is at most:
$$\delta \left ( \frac{1}{1-\delta} \right ) +
\left ( n^{5} {\log  \left ( \frac{1}{\delta+\phi(n)}  \right )}   \right ) ^{-1}.$$
\end{claim}

\begin{proof}
We will analyze the likelihood that $a^\star$ is not selected into $A_{i+1}$, given that it is in $A_i$, for every $i \in\{0,1,\ldots,t(n)\}$.  In every iteration $i$ we can bound the likelihood of $a^\star$ being $\frac{\epsilon}{2}$-underestimated: %using Hoeffding:
\begin{align*}
\Pr \left [\hat{\mu}(a^\star)<\mu(a^\star) - \frac{\epsilon}{2} \right ]
\leq e^{\frac{-\epsilon^2 \ell_{i\text{+}1} }{2}}
%=     e^{- \frac{\epsilon^2}{2} \left (  \frac{2}{\epsilon^2}  \times i\cdot  \log \frac{1}{\delta} \right )  }
=     e^{- (i\text{+}1) \log \frac{1}{\delta} }
=     \delta^{i\text{+}1}
\end{align*}
By definition of \textsc{Aggressive Elimination} $a^\star$ is not in $A_{i+1}$ only if there are at least $|A_i|(\delta +\phi(n))$ arms in $A_i$ whose empirical mean is higher than that of $a^\star$.  By the assumption of the claim, we know that all other arms are $\epsilon$-far from $a^\star$.  If $a^\star$ does not survive to the next round it is because it was $\frac{\epsilon}{2}$-underestimated or at least $|A_{i}|(\delta+\phi(n))$ arms were $\frac{\epsilon}{2}$-overestimated.  By Lemma~\ref{lem:overestimate} we know that the likelihood of more than $|A_i|(\delta + \phi(n))$ arms to be $\frac{\epsilon}{2}$-overestimated is $n^{-6}$.  Thus, by a union bound, in every iteration $i \in \{0,1,\ldots,t(n)\}$ the likelihood of discarding $a^\star$ is at most $\delta^{i+1}+{n}^{-6}$.  The likelihood that $a^\star$ does not survive the last elimination is at most:
\begin{align*}
\sum_{i=0}^{t(n)}\left (\delta^{i \text{+}1} \text{+}  \frac{1}{n^{6}} \right )
 = \left (\sum_{i=0}^{t(n)}\delta^{i\text{+} 1} \right ) \text{+}   \frac{t(n)}{n^{6}}
< \delta \left ( \frac{1}{1-\delta} \right ) \text{+}  \frac{1}{n^{5} \left(\log  \frac{1}{\delta+\phi(n) }\right)      }.
\hspace{0.8in}
\qedhere
\end{align*}

\end{proof}
The main takeaway is that when $n$ is sufficiently large as a function of $\delta$, there is a high probability that $a^\star$ is in the set of arms returned by the procedure when the rest of arms are $\epsilon$-far from $a^\star$.

%In the invocation of \textsc{Aggressive Elimination} we will use some assumptions about $n$ and $\delta$ that will allow

\subsection{A Simple Algorithm under Favorable Conditions}
At this point learning a best arm under favorable conditions seems rather straightforward: we implement \textsc{Aggressive Elimination} and then run \textsc{Na\"ive Elimination} on the remaining set of $\frac{n^{3/4}}{2}$ arms.  We present the algorithm formally below and give details of the analysis in Appendix~\ref{sec:saba_appendix}.
\begin{algorithm}[H]
\caption{\textsc{Simple Approximate Best Arm}}
\label{algo:SABA}
\begin{algorithmic}[1]
    	\INPUT arms $A$, $\epsilon,\delta>0$, noisy oracle for $\mu:A \to [0,1]$
	\STATE $A_{T} \leftarrow \textsc{Aggressive Elimination}(A,\epsilon,\frac{\delta}{2})$
	\OUTPUT \textsc{Na\"ive Elimination}$(A_{T},\epsilon,\frac{\delta}{e})$
	\end{algorithmic}
\end{algorithm}

\begin{restatable}{rClm}{saba}\label{thm:saba}
Assume that there is a unique $\epsilon$-best arm in $A$.  Then $\forall \delta\leq 0.05$ and ${n\geq \max\{1/\delta^4,10^5\}}$, \textsc{SABA} $(\epsilon,\delta)$-learns a best arm with sample complexity $\frac{4n}{\epsilon^2}\log \frac{1}{\delta}$.
\end{restatable}

%\begin{claim}
%Assume there is a unique $\epsilon$-best arm $a^\star$ in $A$ and that $n\geq \max\{1/\delta,10^\frac{5}{2}\}$.  Then for any $\delta\leq 0.05$ \textsc{SABA} $(\epsilon,\delta)$-learns a best arm with sample complexity $\frac{3n}{\epsilon^2}\log \frac{1}{\delta}$.
%\end{claim}

%\newpage

\section{APPROXIMATE BEST ARM ALGORITHM}\label{sec:aba}
In this section we present the Approximate Best Arm (\textsc{Aba}) algorithm which is a modification of \textsc{Saba}.  We first discuss how to remove the assumptions \textsc{Saba} makes and then describe the algorithm.

\paragraph{Removing $n\geq \max\{1/\delta^4,10^{5}\}$ assumption.}
When we seek a bound that holds for any $n$ (i.e. not an asymptotic bound for $n\to \infty$) we avoid this assumption by simply running $\textsc{Na\"ive Elimination}$ when the parameters do not respect these conditions.  It is easy to verify that when $n<1/\delta^4$ or $n<10^5$ and $\delta<0.05$ we can $(\epsilon,\delta)$-learn a best arm by running $\textsc{Na\"ive Elimination}(A,\epsilon,\frac{\delta}{n})$ and the sample complexity is then $\frac{10 n}{\epsilon^2}\log \frac{1}{\delta}$.  When we analyze the asymptotic result in Section~\ref{sec:asymptotic}, we'll show a different modification of the algorithm that doesn't require running $\textsc{Na\"ive Elimination}$.

\paragraph{Removing the unique $\epsilon$-best arm assumption.}
To avoid this assumption we will slightly decrease $\epsilon$ and apply $\textsc{Aggressive Elimination}$ with $\myepsilon=\alpha\cdot \epsilon$ using $\alpha \in [0,1]$ that we later define.  In addition, we will select a random set of size $\frac{n^{{7}/{8}}}{2}$.  Together, this guarantees that we are likely to have an arm that is $\myepsilon$-close to $a^\star$, either in the random set or the output of $\textsc{Aggressive Elimination}$:
\begin{itemize}
\item We prove a claim similar to Claim~\ref{clm:unique} but under weaker conditions.  Specifically we show that as long as there are fewer than $\frac{n^{3/8}}{4}$ arms that are $\myepsilon$-close to $a^\star$, then with sufficient confidence $a^\star$ will be one of the arms returned in $A_T$;
\item Otherwise, there are more than $\frac{n^{3/8}}{4}$ arms that are $\myepsilon$-close to $a^\star$ and one will surface with overwhelming probability (as a function of $n$) in a random set $R$ of size $\frac{n^{{7}/{8}}}{2}$.  %since the probability that none of the $\myepsilon$-close arms are in $R$ is: 
%
%$$\left (1-\frac{1}{4n^{{1}/{4}}}\right )^{\frac{n^{{3}/{4}}}{4}} %= \left ( \left (1-\frac{1}{n^{\frac{1}{4}}}\right )^{n^{\frac{1}{4}}}\right)^{\frac{n^{\frac{1}{2}}}{4}} 
%\leq e^{-\frac{n^{1/2}}{16}}$$
%
\end{itemize}
Consequently, it is very likely that there is an $\epsilon_0$-close arm either in $A_T$ or in the random set $R$ (or both) and running \textsc{Na\"ive Elimination} with appropriate parameters on $A_T \cup R$ will return an $\epsilon$-best arm with probability at least $1-\delta$.

\paragraph{The algorithm.}  The Approximate Best Arm (\textsc{Aba}) algorithm described below is a modification of \textsc{Saba} that incorporates the modifications discussed above.

\begin{algorithm}[H]
\caption{\textsc{Approximate Best Arm}}
\label{algo:greedy}
\begin{algorithmic}[1]
    	\INPUT arms $A$, $\alpha,\epsilon,\delta>0$, noisy oracle for $\mu:A \to [0,1]$
	\STATE \noindent \textbf{initialize} $R\leftarrow \frac{n^{7/8}}{2}$ arms selected u.a.r. 
	%\STATE \textbf{if} {$n < \max\{10^{5},\frac{1}{\delta^4}\}$}
	\STATE \textbf{if} {$n < \max\{10^{5},\delta^{-4}\}$}
	 {\textbf{output} $\textsc{Na\"ive Elimination}(A,\epsilon,\delta)$}	
	%\ENDIF	
	\STATE $A_{T} \leftarrow \textsc{Agressive Elimination}(A,\alpha\cdot\epsilon,\frac{\delta}{2})$
	\OUTPUT $\textsc{Na\"ive Elimination}(A_{T}\cup R,(1 - \alpha)\epsilon,\frac{\delta}{e})$ 
	\end{algorithmic}
\end{algorithm}

We first generalize Claim~\ref{clm:unique} for the case in which there isn't necessarily a unique $\epsilon$-best arm $a^\star$ but rather at most $\frac{n^{3/8}}{4}$ arms that are $\epsilon$-close to $a^\star$.  The proof is similar and deferred to Appendix~\ref{sec:aba_appendix}.

\begin{restatable}{rClm}{superclaim}\label{clm:some_are_far}
Suppose that there are at most $\frac{n^{3/8}}{4}$ arms that are $\epsilon$-close to $a^\star$ in $A$ and the rest are $\epsilon$-far.  Then, the likelihood that $\textsc{Aggressive Elimination}(A,\epsilon,\delta)$ returns a set of arms $A_{t(n)+1}$ that does not contain $a^\star \notin A_{t(n)+1}$ is at most:
$$   \delta \left ( \frac{1}{1-\delta} \right )  +  \left (  n \log \left( \frac{1}{\delta+\phi(n) } \right) \right )^{-1}  .$$
\end{restatable}

We now state the approximation and confidence of $\textsc{Aba}$.  We provide proof sketches that are devoid of some of the calculations, and give full proofs in Appendix~\ref{lem:main_apx_appendix}.

\begin{restatable}{rLem}{approximation}\label{lem:main_apx}
For any $\delta \leq 0.05$ $\textsc{Aba}$ initialized with $\alpha=1-1/e$ returns an $\epsilon$-best arm w.p. $\geq 1-\delta$.
\end{restatable}

\begin{proof}[Sketch]
If $n < \max\{ 1/\delta^4,10^5 \}$ we invoke  $\textsc{Na\"ive Elimination}$ which is guaranteed to return an $\epsilon$-best arm with confidence $1-\delta$.  Otherwise, we assume that $n \geq \max\{ 1/\delta^4,10^5 \}$ and we can analyze the performance of \textsc{Aggressive Elimination} invoked with $\alpha\epsilon$ and $\delta'=\delta/2$.

In the case that there are at most $\frac{n^{3/8}}{4}$ arms that are $\alpha\epsilon$-close to $a^\star$ then according to Claim~\ref{clm:some_are_far} \textsc{Aggressive Elimination} invoked with $\alpha\epsilon$ and $\delta'=\delta/2$ will include $a^\star$ in $A_{T}$ w.p. at least:
\begin{align*}
\delta' \left (\frac{1}{1-\delta'} \right) + \left (  n \log \left( \frac{1}{\delta'+\phi(n) } \right) \right )^{-1} < (1-1/e)\delta \end{align*}
Conditioned on $a^\star \in A_{T}$ the likelihood that $\textsc{Na\"ive Elimination}$ on $A_{T}\cup R$ with approximation $(1-\alpha)\epsilon<\epsilon$ does not return an $\epsilon$-best arm is at most $\delta/e$.  Thus, if there are at most $\frac{n^{3/8}}{4}$ arms that are $\alpha\epsilon$-close to $a^\star$  the algorithm terminates with an $\epsilon$-best arm with probability at least $1-\delta$. 

Otherwise, there are at least $\frac{n^{3/8}}{4}$ arms that are $\alpha\epsilon$-close to $a^\star$.  Since we select arms to $R$ u.a.r. and $|R| = \frac{n^{7/8}}{2}$ the likelihood of not having any arms that are $\alpha\epsilon$-close in $R$ is smaller than $(1-1/e)\delta$.  Let $\tilde{a}$ be an arm that is $\alpha\epsilon$-close to $a^\star$ in $R$.  When we run  $\textsc{Na\"ive Elimination}$ with approximation $(1-\alpha)\epsilon$ and $\delta/e$, we are guaranteed that with probability at least $1-\delta/e$ no arm that is $\epsilon$-far from $a^\star$ will have empirical mean higher than that of $\tilde{a}$.  Since $\tilde{a}$ is $\alpha\epsilon$-close to $a^\star$ and $\alpha< 1$ this implies that the algorithm returns an arm that is at least $\epsilon$-close to $a^\star$ w.p. at least $1-\delta$ in this case as well.
\end{proof}

%\begin{restatable}{rLem}{samplecomplexity}\label{lem:main_samples}
%For any $\delta \leq 0.05$ $\textsc{Aba}$ initialized with $\alpha=1-1/e$ has sample complexity at most:
%%
%$$ \frac{18n}{\epsilon^2}  \log\frac{1}{\delta}.$$
%%
%\end{restatable}

\begin{theorem}\label{thm:main}
For any $\delta \leq 0.05$ $\textsc{Aba}$ initialized with $\alpha=1-1/e$ returns an $\epsilon$-best arm w.p. at least $1-\delta$ using total number of samples of at most:
$18 \times \frac{n}{\epsilon^2}  \log\frac{1}{\delta} .$
\end{theorem}

\begin{proof}[Sketch]
If $n < 1/\delta^4$ or $n< 10^5$ we invoke $\textsc{Na\"ive Elimination}$ and its sample complexity is $\frac{10n}{\epsilon^2}\log\frac{1}{\delta}$.  According to Claim \ref{clm:aggressive} the sample complexity of \textsc{Aggressive Elimination} with approximation $\alpha\epsilon$ and confidence $\delta^{1+c}$ the sample complexity is:
\begin{align}
\frac{1}{\alpha^2} \left ( \frac{2n}{\epsilon^2}\log \frac{1}{\delta}\left (   (1+c)\left ( 1+G(n,\delta^{1+c})\right) \right )  \right ) \label{eq:sc_ae}
\end{align}
For any $\delta<0.05$ we have that $\delta^{1+c}<\delta/2$ for $c=1/4$.  Thus, since we ran \textsc{Aggressive Elimination} with confidence $\delta/2$ and $\alpha=1-1/e$ the sample complexity is at most:
\begin{align*}
 \frac{1}{\alpha^2} \left ( \frac{2n}{\epsilon^2}\log \frac{1}{\delta}\left (   (1+c)\left ( 1+G(n,\delta^{1+c})\right) \right )  \right ) 
& <  8  \left ( \frac{n}{\epsilon^2}\log \frac{1}{\delta}    \right )
\end{align*}

For the sample complexity of the $\textsc{Na\"ive Elimination}$ notice that it is applied on $B = A_{T} \cup R$. Since $\alpha = 1-1/e$ and $|B| = \frac{n^{3/4}}{2}+ \frac{n^{7/8}}{2}$, the sample complexity of $\textsc{Na\"ive Elimination}$ is:
\begin{align}
%\frac{1}{(1-\alpha)^2}\left ( \frac{5}{8}\cdot  \frac{2 n^{7/8}}{\epsilon^2}\log \left (\frac{5}{8}\cdot \frac{n^{7/8}}{\delta} \right ) \right)
\frac{1}{(1-\alpha)^2}\left (  \frac{2 |B|}{\epsilon^2}\log \left ( \frac{|B|}{\delta} \right ) \right)
 <  10\left (\frac{n}{\epsilon^2} \log \frac{1}{\delta} \right)
\end{align}
Therefore, the sample complexity of \textsc{Aggressive Elimination} and $\textsc{Na\"ive Elimination}$ is $  18\times\frac{n}{\epsilon^2}  \log\frac{1}{\delta} $ and the total sample complexity is bounded by %
%\begin{equation*}
   $\frac{18n}{\epsilon^2}  \log\frac{1}{\delta}$ .
%\end{equation*}
%
\end{proof}

\subsection{Asymptotic Sample Complexity}\label{sec:asymptotic}
In our exposition of $\textsc{Aba}$ above, we fixed some parameters to show that it achieves low sample complexity for any value of $n$.  This sample complexity is due (1) \textsc{Na\"ive Elimination} to ensure that $n>\max\{10^5,1/\delta^4\}$ and (2) a convex combination of \textsc{Aggressive Elimination} and  \textsc{Na\"ive Elimination} applied on a sublinear number of arms $A_{T} \cup R$.  Intuitively, to remove (1), if we allow $n$ grow, we can remove the \textsc{Na\"ive Elimination} procedure.  For (2) Recall from Claim~\ref{clm:aggressive} that the sample complexity of \textsc{Aggressive Elimination} is:
$$ \Big(1+ G(n,\delta) \Big ) \times  \frac{2n}{\epsilon^2} \log \frac{1}{\delta}  .$$
Since $\lim_{n\to\infty,\delta\to 0}G(n,\delta)=0$, this converges to sample complexity of $\frac{2n}{\epsilon^2} \log \frac{1}{\delta}$.  What remains is the \textsc{Na\"ive Elimination} applied on a sublinear number of arms $A_{T} \cup R$.  Intuitively, since the number of arms is sublinear in $n$, as $n$ grows the sample complexity converges to 0. 
We elaborate on the asymptotic results in Appendix \ref{sec:asymptotic_appendix} and prove the following theorem.

\begin{theorem}\label{thm:aba_limit}
For any $\lambda>0$ there exist $\delta_0$ and $n_0$ s.t. for any $\delta<\delta_0$ and $n\geq n_0$,  \textsc{Aba} $(\epsilon,\delta)$-learns a best arm with sample complexity at most:
$\Big (2+\lambda \Big )\frac{n}{\epsilon^2}\log \frac{1}{\delta}.$

\end{theorem}

%\newpage

\section{APPROXIMATE BEST ARM BY HOEFFDING}\label{sec:abaleh}

We now describe the Approximate Best Arm Likelihood Estimation (\textsc{Abaleh}) algorithm.  This algorithm is a variant of \textsc{Aba} which achieves a sample complexity that is arbitrarily close to that of $(\epsilon,\delta)$-learning the mean of every arm.  Unlike \textsc{Aba} here we must assume that $n\geq 1/\delta$. 

In this algorithm, we want to circumvent the barrier of $2\times\frac{n}{\epsilon^2}  \log \frac{1}{\delta}$ of \textsc{Aba } and get to the complexity of $(1 + \lambda) \times   \frac{n}{2\epsilon^2}\log \frac{1}{\delta}$ for arbitrarily small $\lambda>0$. 
The main idea is that to determine that one arm is better than the other (assuming they are $\epsilon$-far) it is also possible to estimate one of them to accuracy $(1 - \zeta) \epsilon$ and the other to accuracy $\zeta \epsilon$ for $\zeta >0$ that we choose later. We sample each arm $(1 + \frac{\lambda}{2})\frac{1}{2\epsilon^2}\log \frac{1}{\delta}$ times, but in the analysis we apply a different Hoeffding bound per arm:
\begin {enumerate}
\item For the best arm, in the analysis we apply a Hoeffding bound with accuracy $(1 - \zeta)\epsilon$ and failure probability $\ll \delta $.  This ensures the best arm is approximated up to almost $\epsilon$;
\item For any other arm we apply Hoeffding with accuracy $\zeta \epsilon$, and failure probability $\gg \delta$. The number of samples on each arm is still bounded by $(1 + \frac{\lambda}{2})\frac{1}{2\epsilon^2} \log \frac{1}{ \delta}$, as we pay for the additional accuracy with higher failure probability. This is where we need $\delta$ to be small.
\end{enumerate}

Note that we do not assume the algorithm knows which is the best arm, but the analysis can apply different theorems to different arms.
Since there are $n-1$ arms which are not the best arm, and $n$ is large, we can know how many of them failed the Hoeffding bound. As long as this number is not too large (say $0.001 n$) we can be sure that the best arm moves the next stage with high probability.  
To choose $\zeta$, notice that if there were only two arms, it would be wise to choose $\zeta = 1/2$, but for an arbitrary number of arms we use a smaller $\zeta$ and take $\zeta = 1 - (1 - \frac{\lambda}{16})\sqrt{1 - \frac{\lambda}{8}}$ where $\lambda$ is a parameter of the algorithm.
We defer the proofs to Appendix~\ref{sec:abaleh_appendix}.
\begin{algorithm}[H]
\caption{\textsc{Approximate Best Arm Likelihood Estimation by Hoeffding}}
\label{algo:greedy}
\begin{algorithmic}[1]
    	\INPUT $\epsilon,\delta,\lambda \in (0,1)$, arms $A$, noisy oracle for $\mu:A \to [0,1]$
    \STATE $\alpha \leftarrow \sqrt{1 - \frac{\lambda}{8}}$
    \STATE $R \leftarrow$ a random set of $n^{3/4}$ arms
	\STATE apply $(1+\frac{\lambda}{2})( \frac{1}{2\epsilon^2} \log \frac{1}{\delta}  )$ samples $\forall a \in A$
	\STATE $A_{0} \leftarrow$  the $\frac{\lambda n}{50}$ highest estimated arms in $A$
    \STATE $A_{T} \leftarrow \textsc{Agressive Elimination}(A_0,\epsilon\alpha,\frac{\delta}{4})$
	\OUTPUT $\textsc{Na\"ive Elimination}(A_{T}\cup R,(1\text{-}\alpha)\epsilon,\frac{\delta}{4})$
	\end{algorithmic}
\end{algorithm}

%\begin{restatable}{rLem}{abalehsampeles}\label{single:it}
%Suppose $\lambda < 1$, $\delta \le \delta_0$ where $\delta_0$ is the solution to $\frac{\lambda}{100}= \delta_0^{\lambda^2/256}$, and $n > 1/ \delta$. %Let $a^{\star}$ denote the arm with highest expected mean.
%If there are at most $n^{2/3}$ arms which are $\alpha\epsilon$-close to $a^{\star}$ then w.p. at least $1 - \frac{\delta}{2}$ we have that $a^\star$ is one of the $\frac{\lambda n}{50}$ highest estimated arms in $A$.
%\end{restatable}

\begin{restatable}{rLem}{abalehsampeles}\label{single:it}
Suppose $\lambda < 1$, $\delta \le \delta_0$ where $\delta_0$ is the solution to $\frac{\lambda}{100}= \delta_0^{\lambda^2/256}$, and $n > 1/ \delta$. %Let $a^{\star}$ denote the arm with highest expected mean.
If there are at most $n^{2/3}$ arms which are $\alpha\epsilon$-close to $a^{\star}$ then w.p. at least $1 - \frac{\delta}{2}$ we have that $a^\star$ is one of the $\frac{\lambda n}{50}$ highest estimated arms in $A$.
\end{restatable}

Given Lemma~\ref{single:it}, the proof now follows in a similar manner to previous proofs by bounding the sample complexity and approximation and confidence of all sub procedures.  

\begin{theorem}
For any given $\lambda<1$ there is a $\delta_0$ s.t. for any $\delta < \delta_0$ and $n > 1/\delta$ \textsc{Abaleh} $(\epsilon,\delta)$-learns a best arm with sample complexity at most:
$$\Big (1+\lambda \Big )\frac{n}{2\epsilon^2}\log \frac{1}{\delta}.$$
\end{theorem}

%
%\begin{figure}
%	\centering
%	\includegraphics[width=0.7\linewidth]{plots/abale_fig}
%	\caption{The dots represent the means of the different arms. For all arms, with probability $\gg \delta$ (but still small), the estimation is not in the range $\zeta\epsilon$ (dark blue) and with probability $\ll \delta $ it is not in the range $(1-\zeta)\epsilon$ (light blue). We use the dark interval guarantee and concentration arguments to make sure we do not keep too much bad arms, and the light interval guarantee to make sure that we do keep the best arm.}
%	\label{fig:abalefig}
%\end{figure}
%

\newcommand{\ldesq}[1]{\nicefrac{#1\log(\nicefrac{1}{\delta})}{\epsilon^2}}

\section{LOWER BOUND}\label{sec:lower}
We now consider the family of \emph{elimination algorithms} denoted $\cal{F}$ and defined as follows. An algorithm is in $\cal F$ if it begins when $S=A$ is the set of all possible arms and then: (i) pulls each arm in $S$ once (ii) eliminates some of the arms in $S$, and (iii) if $|S|=1$ terminate, else, go back to (i).
%\begin{enumerate}
%	\item Pull each arm in $S$ once
%	\item Eliminate some of the arms in $S$.
%	\item If $|S|=1$ we are done, else, go back to 1.
%\end{enumerate}
% the algorithm pulls each arm in $S$ once, and discards out some of the arms in $S$. The algorithm is repeated until $|S|=1$.  

Since best arm algorithms have very little degrees of freedom many of them are elimination algorithms. Essentially, the only limitation here is that the algorithm's decisions are irrevocable: if the algorithm considers an arm to be suboptimal and discards it from consideration, it cannot revoke and decision and consider the arm again.

\newcommand{\ldesqb}{\left(\frac{1}{2} - \beta \right)\frac{  n }{\epsilon^2}\log\frac{1}{\delta}}
\newcommand{\almosthalf}{\left(\frac{1}{2} - \beta \right)}

\begin{theorem}
	For every $\beta>0$ there exist $\epsilon_0,\delta_0$ such that for any algorithm in $\mathcal{F}$ which finds an $\epsilon$ best arm with success probability $1 - \delta$ where $\epsilon < \epsilon_0$, $\delta < \delta_0$, there exist $n_0$ such that if $n >n_0$, the algorithm requires at least $\ldesqb $ queries.
	
\end{theorem}

\begin{proof}
	
	Suppose that there exists some algorithm $\cal{A} \in \cal{F}$ which uses less than  $\ldesqb$ queries. Then it must be that after $(1 + \sm)\times \frac{1}{\epsilon^2}\almosthalf \log\frac{1}{\delta} $ iterations, $|S| \le \frac{n}{1+\sm}$. But this means that $\cal{A}$ can succeed with the following task, with probability at least $1 - \delta$:
	
	Given $m =(1 + \sm) \frac{1}{\epsilon^2}\almosthalf \log\frac{1}{\delta}/ $ samples on each arm, choose $\frac{1}{1+\sm}$ of the arms, such that this set contains an $\epsilon $ best arm. We will use $\sm = 0.0001\beta$.

	Consider the following distribution: A bad arm is $0$ w.p. $\frac{1}{2}$. and $1$ $w.p.$ $\frac{1}{2}$. A good arm is $0$ w.p. $p = \frac{1}{2} - \epsilon$ and $1$ w.p. $1-p$. There are $n-1$ bad arms, and one good arm. Hence, $\cal{A}$ needs to identify $\frac{\sm}{1+\sm}$ of the arms, such that the good arm will not be in this set.

	The optimal policy for this task given $(1+\sm)m$ samples on each arm which maximizes the success probability, is to look at the number of zeroes each arm has, and to predict that the $\frac{\sm n}{1+\sm}$ arms which have the largest number of zeroes do not include the good arm. But the success probability of this policy can be bounded as follows:

	For any $\xi >0$ there exists $n_0$ such that if $n > n_0$ w.p. $1 - \xi$ there are at most $\frac{\sm n}{2}$ bad arms with more than $(1 + 0.001\beta\epsilon)\frac{m}{2}$ zeroes. We use $\xi = \frac{\delta}{2}$, which is easily satisfied by $n_0 = \frac{1000}{\beta^2\epsilon^2 \delta^2}$.

	Let $X_G$ be a random variable which denotes the number of zeroes of the good arm. 
	We now bound the probability that the good arm will have too many zeros. That is, $\Pr\left[X_G > k\right]$ where $k=(1 + 0.001\beta\epsilon)\frac{m}{2}$. $X_G$ is the sum of random binomial variables, so we can apply a reverse tail bound to it.

	According to \cite{slud1977distribution}, for $p \leq 1/2$ and $mp \leq k \leq m(1 - p)$ (which is indeed our case), it holds that 
	$$\Pr\left[X_G > k\right]\geq \Pr\left[Z > \frac{k-mp}{\sqrt{mp(1-p)}}\right]$$ where Z is a normal $(0, 1)$ random variable. 
	
	We use a standard lower bound by \cite{borjesson1979simple} for upper tail of a normal random variables:
	$$\Pr[Z>z]\geq \frac{z}{z^2+1}e^{-\frac{z^2}{2}}.$$
	
	In our parameters, we have that $z = \frac{k-mp}{\sqrt{mp(1-p)}} = \frac{(1+0.0005\beta)m\epsilon}{\sqrt{m(1/4-\epsilon^2)}}= \frac{2+0.001\beta}{1-4\epsilon^2}\epsilon\sqrt{m}$ which for $\epsilon<0.0001\beta$ is more then $ 2\epsilon\sqrt{m} = 2\sqrt{(1 + \sm) \almosthalf \log\frac{1}{\delta}}$. There exist $\delta_1$ such that for $\delta<\delta_1$ we have that $z$ is large enough for the following inequality to hold: 
	%$$\frac{z}{z^2+1}e^{-\frac{z^2}{2}}\geq \frac{1}{10}e^{-\frac{z^2}{2-0.001\beta}}.$$
	$$\frac{z}{z^2+1}e^{-\frac{z^2}{2}}\geq e^{-\frac{z^2}{2-0.001\beta}}.$$
	
	Since $z = \frac{2+0.001\beta}{1-4\epsilon^2}\epsilon\sqrt{m} = \frac{2+0.001\beta}{1-4\epsilon^2}\sqrt{(1 + \sm) \almosthalf \log\frac{1}{\delta}}$, then for $\epsilon<0.0001\beta$, we have that $\frac{z^2}{2-0.001\beta} < (2+0.0001\beta)^2{(1 + \sm) \almosthalf \log\frac{1}{\delta}}/(2-0.001\beta)<(1-\beta)\log\frac{1}{\delta}$.

	Combining the inequalities:
	%$$ \Pr \left[X_G > (1 + 0.001\epsilon)m/2 \right] \geq \frac{1}{10}\delta^{(1-\beta/2)}.$$
	$$ \Pr \left[X_G > (1 + 0.001\epsilon)\frac{m}{2} \right] \geq \delta^{1-\beta}.$$

	However, if $\xi < \frac{\delta}{2}$ there exist $\delta_2$ such that if  $\delta <\delta_2$, we have that	
	
	%\[\frac{1}{10}\delta^{(1-\beta/2)} - \xi > \delta.\]
	\[\delta^{1-\beta} - \xi > \delta.\]
	
	This upper bounds the success probability of any algorithm in $\cal{F}$ making too few queries.
	Hence, for $\delta_0<\min\{\delta_1,\delta_2\}$, $\epsilon_0<0.0001\beta$ and $n_0 = \frac{1000}{\beta^2\epsilon^2\delta^2}$ the theorem holds.
\end{proof}

\bibliographystyle{plain}

\bibliography{best_arm}

\begin{thebibliography}{10}

\bibitem{A61}
Arthur~E. Albert.
\newblock The sequential design of experiments for infinitely many states of
  nature.
\newblock {\em The Annals of Mathematical Statistics}, 32:774--799, 1961.

\bibitem{ABM10}
Jean{-}Yves Audibert, S{\'{e}}bastien Bubeck, and R{\'{e}}mi Munos.
\newblock Best arm identification in multi-armed bandits.
\newblock In {\em {COLT} 2010 - The 23rd Conference on Learning Theory, Haifa,
  Israel, June 27-29, 2010}, pages 41--53, 2010.

\bibitem{borjesson1979simple}
P~Borjesson and C-E Sundberg.
\newblock Simple approximations of the error function q (x) for communications
  applications.
\newblock {\em IEEE Transactions on Communications}, 27(3):639--643, 1979.

\bibitem{BMS09}
S{\'{e}}bastien Bubeck, R{\'{e}}mi Munos, and Gilles Stoltz.
\newblock Pure exploration in multi-armed bandits problems.
\newblock In {\em Algorithmic Learning Theory, 20th International Conference,
  {ALT} 2009, Porto, Portugal, October 3-5, 2009. Proceedings}, pages 23--37,
  2009.

\bibitem{BWV13}
S{\'{e}}bastien Bubeck, Tengyao Wang, and Nitin Viswanathan.
\newblock Multiple identifications in multi-armed bandits.
\newblock In {\em Proceedings of the 30th International Conference on Machine
  Learning, {ICML} 2013, Atlanta, GA, USA, 16-21 June 2013}, pages 258--265,
  2013.

\bibitem{CLTL15}
Wei Cao, Jian Li, Yufei Tao, and Zhize Li.
\newblock On top-k selection in multi-armed bandits and hidden bipartite
  graphs.
\newblock In C.~Cortes, N.~D. Lawrence, D.~D. Lee, M.~Sugiyama, and R.~Garnett,
  editors, {\em Advances in Neural Information Processing Systems 28}, pages
  1036--1044. Curran Associates, Inc., 2015.

\bibitem{C59}
Herman Chernoff.
\newblock Sequential design of experiments.
\newblock {\em Ann. Math. Statist.}, 30(3):755--770, 09 1959.

\bibitem{degenne2019pure}
R{\'e}my Degenne and Wouter~M Koolen.
\newblock Pure exploration with multiple correct answers.
\newblock In {\em Advances in Neural Information Processing Systems}, pages
  14564--14573, 2019.

\bibitem{degenne2019non}
R{\'e}my Degenne, Wouter~M Koolen, and Pierre M{\'e}nard.
\newblock Non-asymptotic pure exploration by solving games.
\newblock In {\em Advances in Neural Information Processing Systems}, pages
  14465--14474, 2019.

\bibitem{DGW02}
Carlos Domingo, Ricard Gavald{\`a}, and Osamu Watanabe.
\newblock Adaptive sampling methods for scaling up knowledge discovery
  algorithms.
\newblock {\em Data Mining and Knowledge Discovery}, 6(2):131--152, 2002.

\bibitem{EMM03}
Eyal Even{-}Dar, Shie Mannor, and Yishay Mansour.
\newblock Action elimination and stopping conditions for reinforcement
  learning.
\newblock In {\em Machine Learning, Proceedings of the Twentieth International
  Conference {(ICML} 2003), August 21-24, 2003, Washington, DC, {USA}}, pages
  162--169, 2003.

\bibitem{GGL12}
Victor Gabillon, Mohammad Ghavamzadeh, and Alessandro Lazaric.
\newblock Best arm identification: A unified approach to fixed budget and fixed
  confidence.
\newblock In {\em NIPS}, pages 3221--3229, 2012.

\bibitem{GK16}
Aur{\'{e}}lien Garivier and Emilie Kaufmann.
\newblock Optimal best arm identification with fixed confidence.
\newblock In {\em Proceedings of the 29th Conference on Learning Theory, {COLT}
  2016, New York, USA, June 23-26, 2016}, pages 998--1027, 2016.

\bibitem{garivier2019non}
Aur{\'e}lien Garivier and Emilie Kaufmann.
\newblock Non-asymptotic sequential tests for overlapping hypotheses and
  application to near optimal arm identification in bandit models.
\newblock {\em arXiv preprint arXiv:1905.03495}, 2019.

\bibitem{hagerup1989optimal}
Torben Hagerup and Christine R{\"u}b.
\newblock Optimal merging and sorting on the erew pram.
\newblock {\em Information Processing Letters}, 33(4):181--185, 1989.

\bibitem{JMNB14}
Kevin~G. Jamieson, Matthew Malloy, Robert~D. Nowak, and S{\'{e}}bastien Bubeck.
\newblock lil' {UCB} : An optimal exploration algorithm for multi-armed
  bandits.
\newblock In {\em Proceedings of The 27th Conference on Learning Theory, {COLT}
  2014, Barcelona, Spain, June 13-15, 2014}, pages 423--439, 2014.

\bibitem{JN14}
Kevin~G. Jamieson and Robert~D. Nowak.
\newblock Best-arm identification algorithms for multi-armed bandits in the
  fixed confidence setting.
\newblock In {\em 48th Annual Conference on Information Sciences and Systems,
  {CISS} 2014, Princeton, NJ, USA, March 19-21, 2014}, pages 1--6, 2014.

\bibitem{KS10}
Shivaram Kalyanakrishnan and Peter Stone.
\newblock Efficient selection of multiple bandit arms: Theory and practice.
\newblock In {\em Proceedings of the 27th International Conference on Machine
  Learning (ICML-10), June 21-24, 2010, Haifa, Israel}, pages 511--518, 2010.

\bibitem{TAS12}
Shivaram Kalyanakrishnan, Ambuj Tewari, Peter Auer, and Peter Stone.
\newblock {PAC} subset selection in stochastic multi-armed bandits.
\newblock In {\em Proceedings of the 29th International Conference on Machine
  Learning, {ICML} 2012, Edinburgh, Scotland, UK, June 26 - July 1, 2012},
  2012.

\bibitem{KKS13}
Zohar~Shay Karnin, Tomer Koren, and Oren Somekh.
\newblock Almost optimal exploration in multi-armed bandits.
\newblock In {\em Proceedings of the 30th International Conference on Machine
  Learning, {ICML} 2013, Atlanta, GA, USA, 16-21 June 2013}, pages 1238--1246,
  2013.

\bibitem{katz2019true}
Julian Katz-Samuels and Kevin Jamieson.
\newblock The true sample complexity of identifying good arms.
\newblock {\em arXiv preprint arXiv:1906.06594}, 2019.

\bibitem{KK13}
Emilie Kaufmann and Shivaram Kalyanakrishnan.
\newblock Information complexity in bandit subset selection.
\newblock In {\em {COLT} 2013 - The 26th Annual Conference on Learning Theory,
  June 12-14, 2013, Princeton University, NJ, {USA}}, pages 228--251, 2013.

\bibitem{K84}
Robert Keener.
\newblock Second order efficiency in the sequential design of experiments.
\newblock {\em j-ANN-STAT}, 12(2):510--532, June 1984.

\bibitem{KS63}
J.~Kiefer and J.~Sacks.
\newblock Asymptotically optimum sequential inference and design.
\newblock {\em The Annals of Mathematical Statistics}, 34(3):705--750, 1963.

\bibitem{MT03}
Shie Mannor and John~N. Tsitsiklis.
\newblock Lower bounds on the sample complexity of exploration in the
  multi-armed bandit problem.
\newblock In {\em Computational Learning Theory and Kernel Machines, 16th
  Annual Conference on Computational Learning Theory and 7th Kernel Workshop,
  COLT/Kernel 2003, Washington, DC, USA, August 24-27, 2003, Proceedings},
  pages 418--432, 2003.

\bibitem{R16}
Daniel Russo.
\newblock Simple bayesian algorithms for best arm identification.
\newblock In {\em Proceedings of the 29th Conference on Learning Theory, {COLT}
  2016, New York, USA, June 23-26, 2016}, pages 1417--1418, 2016.

\bibitem{SJR16}
Max Simchowitz, Kevin~G. Jamieson, and Benjamin Recht.
\newblock Best-of-k-bandits.
\newblock In {\em Proceedings of the 29th Conference on Learning Theory, {COLT}
  2016, New York, USA, June 23-26, 2016}, pages 1440--1489, 2016.

\bibitem{STK16}
Adish Singla, Sebastian Tschiatschek, and Andreas Krause.
\newblock Noisy submodular maximization via adaptive sampling with applications
  to crowdsourced image collection summarization.
\newblock In {\em Proceedings of the Thirtieth {AAAI} Conference on Artificial
  Intelligence, February 12-17, 2016, Phoenix, Arizona, {USA.}}, pages
  2037--2043, 2016.

\bibitem{slud1977distribution}
Eric~V. Slud.
\newblock Distribution inequalities for the binomial law.
\newblock {\em The Annals of Probability}, 5(3):404--412, 1977.

\bibitem{UCFN13}
Tanguy Urvoy, Fabrice Clerot, Raphael F{\'e}raud, and Sami Naamane.
\newblock Generic exploration and k-armed voting bandits.
\newblock In {\em Proceedings of the 30th International Conference on
  International Conference on Machine Learning - Volume 28}, ICML'13, pages
  II--91--II--99. JMLR.org, 2013.

\end{thebibliography}

%\end{document}
\newpage

\appendix

\section{Simple Approximate Best Arm Algorithm}\label{sec:saba_appendix}

\begin{claim}
The sample complexity of \textsc{Na\"ive Elimination} is $\frac{2|A|}{\epsilon^2}\log\frac{|A|}{\delta}$ and it returns an arm that is $\epsilon$-close to $a^\star$ with probability at least $1-\delta$.  
\end{claim}

\begin{proof}
To see this, suppose that $a^\star$ is not returned by \textsc{Na\"ive Elimination}. If another arm is returned that is $\epsilon$-close to $a^\star$ then we are done.  Otherwise, assume that \textsc{Na\"ive Elimination} returns an arm $a$ that is $\epsilon$-far. Since any arm is sampled $\frac{1}{2(\epsilon^2/2)}\log\frac{1}{\delta}$ times, by the Hoeffding bound we know that the likelihood of either $\frac{\epsilon}{2}$-underestimating $a^\star$ or $\frac{\epsilon}{2}$-overestimating an $\epsilon$-far arm is $\frac{\delta}{|A|}$.  There are at most $|A|-1$ arms that are $\epsilon$-far from $a^\star$.  By a union bound, $a^\star$ is not $\frac{\epsilon}{2}$ underestimated and none of the $\epsilon$-far arms are $\frac{\epsilon}{2}$-overestimated w.p. at least $1-\delta$.  Thus $a^\star$ has larger empirical mean than any of the $\epsilon$-far arms, implying that the procedure returns an $\epsilon$-best arm w.p. at least $1-\delta$.
\end{proof}

%% The overestimate lemma: %%

\overestimate*\label{lem:overestimate_app}

\begin{proof}
In every iteration $i \in\{0,1,\ldots,t(n)\}$ the likelihood of arm $a \in A$ being $\frac{\epsilon}{2}$-overestimated is:
\begin{align*}
\Pr \left [\hat{\mu}(a)>\mu(a) + \frac{\epsilon}{2} \right ]
\leq e^{\frac{-\epsilon^2 \ell_{i\text{+}1} }{2}}
=     e^{- (i\text{+}1) \log \frac{1}{\delta} }
=     \delta^{i\text{+}1}
\end{align*}
Therefore, in expectation, there are $|A_i| \cdot \delta^{i+1}$ arms that are $\frac{\epsilon}{2}$-overestimated.  Let $X_{a}$ denote the random variable that indicates whether arm $a$ is $\frac{\epsilon}{2}$ overestimated, $X = \sum_{a\in A_i}X_{a}$ and $\hat{X}$ be the number of arms that are $\frac{\epsilon}{2}$-overestimated at iteration $i$.  Again, by Hoeffding, the likelihood of more than $|A_i|(\delta+\phi(n))$ being $\frac{\epsilon}{2}$-overestimated: 

\begin{align}
\Pr \left [ |A_i| \cdot  (\delta + \phi(n)) \textrm{ arms $\frac{\epsilon}{2}$-overestimated} \right ] 
& = \Pr \left [ \hat{X } - \E [ X]  \geq (\delta + \phi(n))|A_i| - \E[X]  \right  ] \\
& = \Pr \left [ \hat{X } - \E [ X]  \geq (\delta - \delta^{i+1} + \phi(n))|A_i| \right  ]  \label{eq:mean} \\
& \leq \Pr \left [ \hat{X } - \E [ X]  \geq \phi(n)|A_i| \right  ]   \label{eq:delta}\\
& \leq \exp \left (-2\phi(n)^2 |A_i| \right )  \\
& \leq \exp  \left (-\phi(n)^2 n^{3/4} \right ) \label{eq:ai}\\
& = \frac{1}{n^6} \label{eq:phi1}
\end{align}
In (\ref{eq:mean}) we use the fact that $\E [X] = |A_i|\cdot \delta^{i+1}$, in (\ref{eq:delta}) we used the fact that $\delta\leq 1$, in (\ref{eq:ai}) we used the fact that there are at least $\frac{n^{3/4}}{4}$ arms in $A_i$, and in (\ref{eq:phi1}) we used the definition of $\phi(n)$ in (\ref{eq:phi_dfn}).
 \end{proof}

\saba*

\begin{proof}
The proof follows from the sample complexity and approximation and confidence of \textsc{Aggressive Elimination} and \textsc{Na\"ive Elimination}.  The sample complexity is the total number of samples required to implement \textsc{Aggressive Elimination} with $\delta/2$ and \textsc{Na\"ive Elimination} on the remaining arms with $\delta/e$.  A convenient way to express the sample complexity of \textsc{Aggressive Elimination} is to use a constant $c$ for which $\delta^{1+c} = \delta/2$.  The sample complexity of \textsc{Aggressive Elimination} with $\delta^{1+c}$ is:
\begin{equation}\label{eq:saba1}
 (1+c)\left ( 1+G(n,\delta^{1+c})\right) \times \frac{2n}{\epsilon^2}\log \frac{1}{\delta}
\end{equation}
In our case, if we assume that $\delta<0.05$ then for $c = 1/4$ we get that $\delta^{1+c} < \delta/2$.

For  \textsc{Na\"ive Elimination} executed on $\frac{n^{3/4}}{2}$ arms with $\delta/e$ the sample complexity is:
\begin{equation}\label{eq:saba2}
\left (  \frac{1}{2n^{\frac{1}{4}}} \left ( 1 + \frac{3\log n+4}{4\log\frac{1}{\delta}}\right ) \right ) \times \frac{2n}{\epsilon^2}\log \frac{1}{\delta}
\end{equation}
For $n\geq 10^5$ and $\delta<0.05$ the sample complexity of \textsc{Saba} is $(\ref{eq:saba1})+(\ref{eq:saba2}) <\frac{4n}{\epsilon^2}\log \frac{1}{\delta}$.

In terms of approximation and confidence, for $n\geq 10^5$ then $\phi(n)<0.12$ and for $\delta<0.05$ we get $\log(\frac{1}{\delta'+\phi(n)}) > 1$.  Applying \textsc{Aggressive Elimination} on $\delta' = \delta/2$ when $\delta\leq 0.05$ implies that $a^\star$ is not in $A_{T}$ w.p. at most: 
\begin{align*}
&  \delta'\left (\frac{1}{1- \delta'} \right ) + \frac{1}{n^{5}\log (\frac{1}{\delta' + \phi(n)})}
%& <   \frac{40 }{39} \delta ' + \frac{1}{10  \sqrt{n} \log ( \frac{12}{7})}\\
 < \frac{20 \delta}{39} + \frac{1}{n^{5}}
%& < \frac{40 \delta}{78} + \frac{6\delta}{100}\\
 < \left (1-\frac{1}{e} \right)\delta
\end{align*}
Finally, assuming that $a^\star$ is in $A_{T}$ then the probability it is not returned by \textsc{Na\"ive Elimination}  is at most $\delta/e$.  By union bound, the likelihood that $a^\star$ is either not in $A_{t+1}$ or not selected by \textsc{Na\"ive Elimination} is at most $\delta$.
\end{proof}

\section{Approximate Best Arm Algorithm}\label{sec:aba_appendix}

\superclaim*

\begin{proof}
Since there are at most $\frac{n^{3/8}}{4}$ arms that are $\epsilon$-close, we know that in every iteration $i \in\{0,1,\ldots,t(n)\}$ there are at least $|A_i| - \frac{n^{3/8}}{4}$ that are $\epsilon$-far from $a^\star$.  In the worst case, in every iteration every one of the $\epsilon$-close arms is overestimated in such a way that its empirical mean is larger than that of $a^\star$.  In this case, the only way that $a^\star$ is not included in round $A_{i+1}$ is if $a^\star$ is $\frac{\epsilon}{2}$-underestimated and there are more than $|A_i|(\delta+\phi(n)) - \frac{n^{3/8}}{4}$ arms that are $\epsilon$-far that are $\frac{\epsilon}{2}$ overestimated.

As in the proof of Claim~\ref{clm:unique} using the Hoeffding bound we know that the likelihood of $a^\star$ being $\frac{\epsilon}{2}$-underestimated is at most $\delta^{i\text{+}1}$.  The likelihood of an arm being $\frac{\epsilon}{2}$-overestimated is at most $\delta^{i\text{+}1}$ and in expectation there are $|A_i|\delta^{i\text{+}1}$ arms that are $\frac{\epsilon}{2}$-overestimated in every iteration $i$.  Let $X_{a}$ denote the random variable that indicates whether arm $a$ is $\frac{\epsilon}{2}$-overestimated, $X = \sum_{a\in A_i}X_{a}$ and $\hat{X}$ be the number of arms that are $\frac{\epsilon}{2}$-overestimated at iteration $i$.  Again, by Hoeffding, the likelihood of more than $|A_i|(\delta+\phi(n)) - \frac{n^{3/8}}{4} $ being $\frac{\epsilon}{2}$-overestimated: 

%\begin{align}
%%\Pr \left [ |A_i| \cdot  \left (\delta + \phi(n) \right) - \frac{n^{3/8}}{4} \textrm{ arms $\frac{\epsilon}{2}$-overestimated} \right ] 
%%& = 
%\Pr \left [ \hat{X }  \geq (\delta + \phi(n))|A_i| -\frac{n^{3/8}}{4}  \right  ] & = 
%\Pr \left [ \hat{X } - \E [ X]  \geq (\delta + \phi(n))|A_i| - \E[X] -\frac{n^{3/8}}{4} \right  ] \\
%& = \Pr \left [ \hat{X } - \E [ X]  \geq (\delta - \delta^{i+1} + \phi(n))|A_i| -\frac{n^{3/8}}{4} \right  ]  \label{eq:mean2} \\
%& \leq \Pr \left [ \hat{X } - \E [ X]  \geq \phi(n)\left (|A_i| -  \frac{n^{3/8}}{4\phi(n)}\right) \right  ]   \label{eq:delta2}\\
%& \leq \Pr \left [ \hat{X } - \E [ X]  \geq \phi(n)\left (|A_i| -  \frac{n^{3/4}}{12}\right) \right  ]   \label{eq:delta3}\\
%& \leq \Pr \left [ \hat{X } - \E [ X]  \geq \phi(n)\left (\frac{n^{3/4}}{2} -  \frac{n^{3/4}}{12}\right) \right  ]  \label{eq:delta4}\\
%& = \Pr \left [ \hat{X } - \E [ X]  \geq \phi(n)\left (\frac{5n^{3/4}}{12} \right) \right  ]   \label{eq:delta5}\\
%& \leq  \exp \left (-\frac{10\phi(n)^2 n^{3/4} }{12}  \right )   \label{eq:ai2}\\
%%& = \exp{\left (\frac{-8 \left ( \log(10) + 2\log(n) + \log\log(n) \right )} {3}         \right) }   \label{eq:ai3}\\
%%& = \exp\left (-\frac{8}{3} ( \log(10) + 2\log(n) + \log\log(n) ) \right )\\
%& = \frac{1}{n^5} \label{eq:phi2}
%\end{align}

\begin{align}
%\Pr \left [ |A_i| \cdot  \left (\delta + \phi(n) \right) - \frac{n^{3/8}}{4} \textrm{ arms $\frac{\epsilon}{2}$-overestimated} \right ] 
%& = 
\Pr \left [ \hat{X }  \geq (\delta + \phi(n))|A_i| -\frac{n^{3/8}}{4}  \right  ] & = 
\Pr \left [ \hat{X } - \E [ X]  \geq (\delta + \phi(n))|A_i| - \E[X] -\frac{n^{3/8}}{4} \right  ] \\
& = \Pr \left [ \hat{X } - \E [ X]  \geq (\delta - \delta^{i+1} + \phi(n))|A_i| -\frac{n^{3/8}}{4} \right  ]  \label{eq:mean2} \\
& \leq \Pr \left [ \hat{X } - \E [ X]  \geq \phi(n)\left (|A_i| -  \frac{n^{3/8}}{4\phi(n)}\right) \right  ]   \label{eq:delta2}\\
& \leq \Pr \left [ \hat{X } - \E [ X]  \geq \phi(n)\left (|A_i| -  \frac{n^{3/4}}{12}\right) \right  ]   \label{eq:delta3}\\
& \leq \Pr \left [ \hat{X } - \E [ X]  \geq \phi(n)\left (\frac{2}{3}|A_i|\right) \right  ]  \label{eq:delta4}\\
& \leq  \exp \left (-\frac{2\cdot4\phi(n)^2 |A_i| }{9}  \right )   \label{eq:ai1}\\
& \leq  \exp \left (-\frac{2\phi(n)^2 n^{3/4} }{9}  \right )   \label{eq:ai2}\\
%& = \exp{\left (\frac{-8 \left ( \log(10) + 2\log(n) + \log\log(n) \right )} {3}         \right) }   \label{eq:ai3}\\
%& = \exp\left (-\frac{8}{3} ( \log(10) + 2\log(n) + \log\log(n) ) \right )\\
& = \frac{1}{n^{4/3}} \label{eq:phi2}
\end{align}
In (\ref{eq:mean2}) we use the fact that $\E [X] = |A_i|\cdot \delta^{i+1}$, in (\ref{eq:delta2}) we used the fact that $\delta\leq 1$, in (\ref{eq:delta3}) we used the fact that $\phi(n)> 3n^{-3/8}$ for $n\geq 5$, 
in {eq:delta4} and (\ref{eq:ai2}) we used the fact that there are at least $\frac{n^{3/4}}{4}$ arms in $A_i$, and in (\ref{eq:phi2}) we used the definition of $\phi(n)$ in (\ref{eq:phi_dfn}).

Having calculated the likelihood that $a^\star$ is $\frac{\epsilon}{2}$-underestimated to be $\delta^{i+1}$ and the likelihood that there are at least $|A_i|(\delta + \phi(n)) - \frac{n^{3/8}}{4}$ arms that are $\frac{\epsilon}{2}$-overestimated, by a union bound, in every iteration $i \in \{0,1,\ldots,t(n)\}$ the likelihood of discarding $a^\star$ is at most:
$$\delta^{i+1} + \frac{1}{n^{4/3}} $$
Taking a union bound over the likelihood that $a^\star$ is discarded in every iteration $i\in\{0,1,\ldots,t(n)\}$ we get that the likelihood that $a^\star$ does not survive the last elimination step is at most:
\begin{align*}
\sum_{i=0}^{t(n)}\left (\delta^{i \text{+}1} \text{+}  \frac{1}{n^{4/3} } \right )
 = \left (\sum_{i=0}^{t(n)}\delta^{i\text{+} 1} \right ) \text{+}   \frac{t(n)}{n^{4/3}}
< \delta \left ( \frac{1}{1-\delta} \right ) \text{+}  \frac{1}{n \left(\log  \frac{1}{\delta+\phi(n) }\right)      }.
\end{align*}

\end{proof}

\approximation*

\begin{proof}\label{lem:main_apx_appendix}
If $n < \max\{ 1/\delta^4,10^5 \}$ we invoke  $\textsc{Na\"ive Elimination}$ which is guaranteed to return an $\epsilon$-best with confidence $1-\delta$.  Otherwise, we assume that $n \geq \max\{ 1/\delta^4,10^5 \}$ and we can analyze the performance of \textsc{Aggressive Elimination} invoked with $\alpha\epsilon$ and $\delta'=\delta/2$.

In the case that there are at most $\frac{n^{3/8}}{4}$ arms that are $\alpha\epsilon$-close to $a^\star$ then according to Claim~\ref{clm:some_are_far} \textsc{Aggressive Elimination} invoked with $\alpha\epsilon$ and $\delta'=\delta/2$ will include $a^\star$ in $A_{T}$ w.p. at least :  
\begin{align*}
\delta' \left (\frac{1}{1-\delta'} \right) + \left (  n \log \left( \frac{1}{\delta'+\phi(n) } \right) \right )^{-1} \end{align*}
When $\delta'=\frac{\delta}{2}$ and $\delta<0.05$ we have that $\delta' \left (\frac{1}{1-\delta'} \right) < \frac{20}{39}\delta$.  When $\delta<0.05$ then $\log \left( \frac{1}{\delta'+\phi(n) } \right) > 1$ and  since $n\geq 1/\delta^4$ we have that  
$$\left (   n \log \left( \frac{1}{\delta'+\phi(n) } \right) \right )^{-1} < {\delta^4}.$$  
Together we have that the likelihood that $a^\star$ is not in $A_T$ returned by \textsc{Aggressive Elimination} is:
\begin{align*}
\delta' \left (\frac{1}{1-\delta'} \right) + \left (  n \log \left( \frac{1}{\delta'+\phi(n) } \right) \right )^{-1} < (1-1/e)\delta \end{align*}
Conditioned on $a^\star \in A_{T}$ the likelihood that $\textsc{Na\"ive Elimination}$ on $A_{T}\cup R$ with approximation $(1-\alpha)\epsilon<\epsilon$ does not return an $\epsilon$-best arm is at most $\delta/e$.  Thus, if there are are at most $\frac{n^{3/8}}{4}$ arms that are $\alpha\epsilon$-close to $a^\star$  the algorithm terminates with an $\epsilon$-best arm with probability at least $1-\delta$. 

Otherwise, there are at least $\frac{n^{3/8}}{4}$ arms that are $\alpha\epsilon$-close to $a^\star$.  Since we select arms to $R$ u.a.r. and $|R| = \frac{n^{7/8}}{2}$ the likelihood of not having any arms that are $\alpha\epsilon$-close in $R$ is at most:
\begin{align*}
\left ( 1 - \frac{n^{3/8}}{4n}\right )^{|R|} 
= \left ( 1 - \frac{1}{4n^{5/8}}\right )^{\frac{n^{7/8}}{2}} 
= \left ( \left ( 1 - \frac{1}{4n^{5/8}}\right )^ {4n^{5/8} } \right )^{\frac{n^{1/4}}{8}}   
< e^{-\frac{n^{1/4}}{8}}
\end{align*}
When we have $n>10^5$ then $e^{-\frac{n^{1/4}}{8}}<\frac{1}{n^{1/4}} (1-1/e) $.  Since $n>1/\delta^4$ we get that the likelihood of an $\alpha\epsilon$-close to $a^\star$ not appearing in $R$ is smaller than $(1-1/e)\delta$.  Let $\tilde{a}$ be an arm that is $\alpha\epsilon$-close to $a^\star$ in $R$.  When we run  $\textsc{Na\"ive Elimination}$ with approximation $(1-\alpha)\epsilon$ and $\delta/e$, we are guaranteed that with probability at least $1-\delta/e$ no arm that is $\epsilon$-far from $a^\star$ will have empirical mean higher than that of $\tilde{a}$.  Since $\tilde{a}$ is $\alpha\epsilon$-close to $a^\star$ and $\alpha< 1$ this implies that the algorithm returns an arm that is at least $\epsilon$-close to $a^\star$ w.p. at least $1-\delta$ in this case as well.
\end{proof}

%\samplecomplexity*
%\begin{restatable}{rLem}{samplecomplexity}\label{lem:main_samples}
%For any $\delta \leq 0.05$ $\textsc{Aba}$ initialized with $\alpha=1-1/e$ has sample complexity at most:
%%
%$$ \frac{18n}{\epsilon^2}  \log\frac{1}{\delta}.$$
%%
%\end{restatable}
\begin{lemma}
For any $\delta \leq 0.05$ $\textsc{Aba}$ initialized with $\alpha=1-1/e$ has sample complexity at most:
$$ \frac{18n}{\epsilon^2}  \log\frac{1}{\delta}.$$
\end{lemma}
\begin{proof}
If $n < 1/\delta^4$ or $n< 10^5$ we invoke $\textsc{Na\"ive Elimination}$ and its sample complexity is $\frac{10n}{\epsilon^2}\log\frac{1}{\delta}$.  To see this, notice that if $n < 1/\delta^4$  the sample complexity of $\textsc{Na\"ive Elimination}$ is:

	$$\frac{2n}{\epsilon^2}\log \frac{n}{\delta} = \frac{2n}{\epsilon^2} \left (\log \frac{1}{\delta} + \log(n) \right) \leq 
	\frac{2n}{\epsilon^2} \left (\log \frac{1}{\delta} + 4\log\frac{1}{\delta} \right) = \frac{10n}{\epsilon^2}\log \frac{1}{\delta} 
	$$

If $n < 10^5$ then when $\delta<0.05$ the sample complexity of $\textsc{Na\"ive Elimination}$ is:
\begin{align*}
\frac{2n}{\epsilon^2}\log \frac{n}{\delta} = \frac{2n}{\epsilon^2}\log \frac{1}{\delta}\left (1 + \frac{\log n}{\log\frac{1}{\delta}} \right )  <  \frac{10n}{\epsilon^2}\log\frac{1}{\delta}.
\end{align*}

According to Claim \ref{clm:aggressive} the sample complexity of \textsc{Aggressive Elimination} with approximation $\epsilon$ and confidence $\delta$ is:
$$ \Big(1+ G(n,\delta) \Big ) \times  \frac{2n}{\epsilon^2} \log \frac{1}{\delta}.$$
Therefore, when running with approximation $\alpha\epsilon$ and confidence $\delta' = \delta^{1+c}$ the sample complexity is:
\begin{align}
\frac{1}{\alpha^2} \left ( \frac{2n}{\epsilon^2}\log \frac{1}{\delta}\left (   (1+c)\left ( 1+G(n,\delta^{1+c})\right) \right )  \right ) \label{eq:sc_ae}
\end{align}
For any $\delta<0.05$ we have that $\delta^{1+c}<\delta/2$ for $c=1/4$.  Thus, since we ran \textsc{Aggressive Elimination} with confidence $\delta/2$ and $\alpha=1-1/e$ the sample complexity is at most:
\begin{align*}
 \frac{1}{\alpha^2} \left ( \frac{2n}{\epsilon^2}\log \frac{1}{\delta}\left (   (1+c)\left ( 1+G(n,\delta^{1+c})\right) \right )  \right ) 
& <  \frac{1}{(1-1/e)^2} \left ( \frac{2n}{\epsilon^2}\log \frac{1}{\delta}\left ( \frac{5}{4}\left ( 1+G(n,\delta^{\frac{5}{4}})\right) \right )  \right )\\
& =  \frac{10}{4(1-1/e)^2} \times \left ( \frac{n}{\epsilon^2}\log \frac{1}{\delta}\left (\left ( 1+G(n,\delta^{\frac{5}{4}})\right) \right )  \right )\\
& <  \frac{10}{4(1-1/e)^2} \times \left ( \frac{n}{\epsilon^2}\log \frac{1}{\delta} \times 1.2   \right )\\
& <  8 \times \left ( \frac{n}{\epsilon^2}\log \frac{1}{\delta}    \right )
\end{align*}

For the sample complexity of the $\textsc{Na\"ive Elimination}$ notice that it is applied on $A_{T} \cup R$ where $|A_{T}| = \frac{n^{3/4}}{2}$ and $|R| = \frac{n^{7/8}}{2}$.  For any $n\geq 10^5$ we have that $\frac{n^{3/4}}{2} < \frac{n^{7/8}}{8}$ and therefore $|A_{T}\cup R| \leq \frac{5}{8}\cdot n^{7/8}$.  Since $\alpha = 1-1/e$ the sample complexity of $\textsc{Na\"ive Elimination}$ is: 
\begin{align}
\frac{1}{(1-\alpha)^2}\left ( \frac{5}{8} \frac{2 n^{7/8}}{\epsilon^2}\log (\frac{5}{8}\cdot \frac{n^{7/8}}{\delta} ) \right)
& <  \frac{1}{(1-\alpha)^2}\left ( \frac{5}{8} \frac{2 n^{7/8}}{\epsilon^2}\log \left (\frac{n^{7/8}}{\delta} \right) \right)\\
& =  \frac{10\cdot e^{2}}{8\cdot n^{1/8}}\left (\frac{n}{\epsilon^2} \log \frac{n^{7/8}}{\delta} \right)\\
& =  \frac{10\cdot e^{2}}{8 \cdot n^{1/8}}\left (\frac{n}{\epsilon^2} \left (\frac{7}{2} \log (n^{-1/4}) + \log \frac{1}{\delta} \right ) \right)\\
& <  \frac{10\cdot e^{2}}{8 \cdot n^{1/8}}\left (\frac{n}{\epsilon^2} \left (\frac{7}{2}\log \frac{1}{\delta} + \log \frac{1}{\delta} \right ) \right)\\
& =  \frac{45\cdot e^{2}}{8 \cdot n^{1/8}}\left (\frac{n}{\epsilon^2} \log \frac{1}{\delta} \right)\\
& \leq  \frac{45\cdot e^{2}}{8\cdot {10}^{5/8}}\left (\frac{n}{\epsilon^2} \log \frac{1}{\delta} \right)\\
& <  10\left (\frac{n}{\epsilon^2} \log \frac{1}{\delta} \right)
\end{align}
Therefore, the sample complexity of \textsc{Aggressive Elimination} and $\textsc{Na\"ive Elimination}$ is smaller then:
\begin{equation*}
18   \times \left (  \frac{n}{\epsilon^2}  \log\frac{1}{\delta} \right).
\end{equation*}
\end{proof}

\subsection{Asymptotic Sample Complexity}\label{sec:asymptotic_appendix}

\paragraph{Generalization of $\phi(n)$.}  Recall that in our algorithm we condition on $n^{1/4}\geq 1/\delta$ and otherwise implement $\textsc{Na\"ive Elimination}(A,\epsilon,\frac{\delta}{n})$.  In general, $\forall d\geq 0$ if $n^{d} < 1/\delta$ we can $(\epsilon,\delta)$-learn the best arm using $\textsc{Na\"ive Elimination}(A,\epsilon,\frac{\delta}{n})$ with sample complexity
\begin{align*}
\frac{2n}{\epsilon^2} \left (\log \frac{1}{\delta} + \frac{\log (n^d)}{d}  \right ) =2 \left (1+\frac{1}{d} \right ) \frac{n}{\epsilon^2}\log \frac{1}{\delta}
\end{align*}
For any choice of $d$ we can modify the \textsc{Aggressive Elimination} to produce the same confidence and approximation guarantees under the assumption that $n^d\geq 1/\delta$, for any $d\in[0,\sqrt{n}]$.  To do so all we need to do is make a modest modification in the definition of $\phi(n)$.  Under an assumption $n^{d}\geq 1/\delta$ our definition of $\phi(n)$ was designed to satisfy the following inequality:
\begin{equation}
\exp \left ( -\phi(n)^2 \left (|A_i| - \frac{n^{3/8}}{4\phi(n)}  \right ) \right ) \leq \frac{1}{10\cdot n^{d}\log n }\label{eq:phi2}
\end{equation}
the left hand expression is the likelihood of the event that in an iteration $i$ the number of arms that are $\epsilon$-far from $a^\star$ that are overestimated is such that $a^\star$ is not included in the next round.  The righthand expression becomes smaller than $\delta/(10 \log(n))$ when $n^d>1/\delta$.   

We can therefore generalize the definition of $\phi(n)$ to $\phi(n,d)$ as follows:
$$\phi(n,d) =   \sqrt{\frac{\log(10) + d\log(n) + \log\log(n)}{n^{3/4}}   }$$
%
%Choosing $d = \sqrt{n}$, for example, gives us the desired properties we need.  
The larger $d$ is so is the sample complexity, but for $d = \sqrt{n}$ we get our desired asymptotic behavior.  In particular get $\lim_{n\to \infty}\phi(n,d) = 0$, thus for any $\delta<1$ we get $\lim_{n\to \infty}G(n,\delta)=0$ and the number of rounds until the algorithm terminates $t(n)=\log n \times \left ( \log \left (\frac{1}{\delta + \phi(n,d)} \right)^{-1}\right)$ approaches $\log n$ as well.

If $n^{\sqrt{n}}<\frac{1}{\delta}$ we may use the \textsc{Aggressive Elimination} and getting a sample complexity of
	$$\frac{2n}{\epsilon^2}\log \frac{n}{\delta} = \frac{2n}{\epsilon^2} \left (\log \frac{1}{\delta} + \log(n) \right) \leq 
\frac{2n}{\epsilon^2} \left (\log \frac{1}{\delta} + \log(\log^2 (\frac{1}{\delta})) \right) 
$$
and there exist $\delta_0$ s.t. if $\delta<\delta_0$, the total sample complexity is  $(2+\lambda)\frac{n}{\epsilon^2}\log \frac{1}{\delta}$ for any $\lambda>0$.

\paragraph{Choosing $\alpha$ as a function of $n$.}
The sample complexity of \textsc{Aba} is a convex combination of the sample complexity of \textsc{Aggressive Elimination} and \textsc{Na\" ive Elimination}:
\begin{align}
\frac{1}{\alpha^2} \left ( \frac{2n}{\epsilon^2}\log \frac{1}{\delta}\left (   (1+c)\left ( 1+G(n,\delta^{1+c})\right) \right )  \right ) 
+ \frac{1}{(1-\alpha)^2}\left ( \frac{5}{8} \frac{2 n^{7/8}}{\epsilon^2}\log \left (\frac{5}{8}\cdot \frac{n^{7/8}}{\delta} \right ) \right)
\end{align}
If we choose $\alpha=(1-n^{-\frac{1}{16}})$ then as $n$ tends to infinity, in the limit the sample complexity is:
\begin{align*}
% \lim_{n \rightarrow \infty} &\left (1\text{-} \frac{1}{n^{\frac{1}{16}}} \right )^2\left ( \frac{2n}{\epsilon^2}\log \frac{1}{\delta}\left (   (1\text{+}c)\left ( 1\text{+}G(n,\delta^{1\text{+}c})\right) \right )  \right )+ \frac{1}{n^{1/8}} \left ( \frac{26}{44} \frac{2 n^{7/8}}{\epsilon^2}\log \left (\frac{26}{44}\cdot \frac{n^{7/8}}{\delta} \right ) \right)   \\= & 
\left (   1 +c \right ) \left ( \frac{2n}{\epsilon^2}\log \frac{1}{\delta} \right)
\end{align*}
Where we relied on the fact that for any fixed $\delta<1$, $\lim_{n\to \infty}G(n,\delta^{1+c})=0$ for any choice of $c>0$.  Recall that we use $c$ to shrink $\delta$ so that instantiating \textsc{Aggressive Elimination} with $\delta^{1+c}$ is guaranteed to include $a^\star$ in its output w.p. at least $(1-1/e)\delta$.  As $\delta$ becomes smaller we require a smaller choice of $c$ as well.  Thus, for any $c$ there exists a $\delta_0$ s.t. for any $\delta<\delta_0$ running \textsc{Aggressive Elimination} with $\delta^{1+c}$ is guaranteed to include $a^\star$ in its output with probability at least $(1-1/e)\delta$.

\begin{theorem}\label{thm:aba_limit_appendix}
	For any $\lambda>0$ there exist $\delta_0$ and $n_0$ s.t. for any $\delta<\delta_0$ and $n\geq n_0$,  \textsc{Aba} $(\epsilon,\delta)$-learns a best arm with sample complexity at most:
	$$\Big (2+\lambda \Big )\frac{n}{\epsilon^2}\log \frac{1}{\delta}.$$
\end{theorem}

\section{Approximate Best Arm Likelihood Estimation by Hoeffding}\label{sec:abaleh_appendix}

\abalehsampeles*

\begin{proof}
	First, we apply a Hoeffding bound on the estimation of $a^\star$. Suppose that we would like to estimate the value of $a^\star$ to accuracy $\epsilon \cdot (1 - \frac{\lambda}{16})\alpha$, with success probability at least $1 - \frac{\delta}{4}$. The number of samples this requires is
	\begin{align*}
	\frac{\log \frac{4}{\delta}}{2\alpha^2\epsilon^2(1 - \frac{\lambda}{16})^2}
	= \frac{\log \frac{4}{\delta}}{2\alpha^2\epsilon^2(1 - \frac{\lambda}{8} + \frac{\lambda^2}{256})}
	\le \frac{ \log \frac{1}{ \delta}}{2\epsilon^2\alpha^2(1 - \frac{\lambda}{8})}
	%& =  \frac{1}{2\epsilon^2(1 -\frac{ \lambda}{8})^2}  \log \frac{1}{ \delta}\\
	\le \frac{ \log \frac{1}{\delta}}{2\epsilon^2(1 - \frac{\lambda}{4})}
	& = \left (1 + \frac{\lambda}{2} \right) \frac{  1 } {2\epsilon^2}\log \frac{1}{\delta}
	\end{align*}
	where the first inequality uses $\delta < \delta_0$ and the second one uses $\lambda < 1$. Since we have taken sufficiently many samples, the Hoeffding inequality applies.
	
	For any other arm, we apply the Hoeffding bound to estimate its mean with accuracy $\epsilon \cdot \frac{\lambda}{16}$, but with failure probability $1 - \delta^{\lambda^2/256}$. The number of samples this requires is:
	\[\frac{ 256}{2\epsilon^2\lambda^2}\log \frac{1}{\delta^{\lambda^2/256}} = \frac{1}{2\epsilon^2} \log\frac{1}{ \delta}\]
	where we took the exponent out of the logarithm. Achieving this approximation and confidence is possible in this case as well since we are actually performing $(1 + \frac{\lambda}{2})\frac{1}{2\epsilon^2}\log\frac{1}{ \delta}$ samples on each arm.
	
	But since $\delta < \delta_0$, and $\frac{\lambda}{100}=\delta_0^{\lambda^2/256}$, this approximation is achievable when failure probability for each arm is bounded from above by $\frac{\lambda}{100}$. Hence the probability that we estimate more than $\frac{\lambda}{80}$ arms incorrectly is exponentially small in $n$. Since $n > 1/\delta$ this failure probability is at most $\frac{\delta}{4}$.
	
	Taking a union bound over both events, we get that with probability at least $1 - \frac{\delta}{2}$ we have that $a^\star$ was estimated up to error $\epsilon(1 - \frac{\lambda}{8})\alpha$, and at most $\frac{\lambda n}{80}$ arms were estimated to error at least $\frac{\alpha\epsilon \lambda}{8}$. Condition on this event.  Now there are two types of arms that we may estimate to be larger than $a^\star$:
	\begin{itemize}
		\item Arms which are $\epsilon\alpha$ close to $a^\star$: there are fewer than $n^{2/3} < \frac{3\lambda n}{400}$ such arms, since $n > \frac{1}{\delta_0}$;
		\item Arms which were estimated incorrectly: there are at most $\frac{\lambda n}{80}$ such arms.
	\end{itemize}
	As $\frac{\lambda n}{80} + \frac{3\lambda n}{400} = \frac{\lambda n}{50}$ and $|A_0| = \frac{\lambda n}{50}$, w.p. $\geq 1 - \frac{\delta}{2}$ the arm $a^\star$ is chosen to $A_0$.
\end{proof}

\begin{lemma}
For any $\lambda \in [0,1]$, $\delta \le \delta_0$ where $\delta_0$ is the solution to $\frac{\lambda}{100}= \delta_0^{\lambda^2/64}$ suppose $n>1/\delta$.  Then \textsc{Abaleh} returns an $\epsilon$ -best arm w.p. at least $1 - \delta$.
\end{lemma}

%\abaleh*

\begin{proof}
Let $G$ denote the set of arms which are $\alpha \epsilon$ close to $a^{\star}$. We consider two cases.  First, if $|G| < n^{2/3}$ then according to Lemma \ref{single:it} with probability at least $1 - \frac{\delta}{2}$ we have $a^{\star} \in A_0$.  Conditioning on this event, note that since $n > 1/ \delta_0$ we have that $n^{2/3} < \big(\frac{\lambda n}{100}\big)^{3/4}$ and hence we can apply Claim \ref{clm:unique} and deduce that with probability $1 - \frac{\delta}{4}$ we have that $A_T$ contains an $\epsilon \alpha$ approximate best arm. Finally, in this case with probability $1 - \frac{\delta}{4}$ we have $\textsc{Na\"ive Elimination}$ finds the an $(1 - \alpha)\epsilon$ approximate best arm to an $\alpha \epsilon$ approximate best arm, which gives an $\epsilon$ best arm as required. Summing the errors and applying a union bound proves the lemma.

 On the other hand, if $|G| \ge n^{2/3}$, then with probability $1 - 2^{-O(n^{1/6})} \ge 1 - \frac{\delta}{2}$ (since $n \ge 1/\delta$) we have that $T \cap G$ will be non empty. Again, with probability at least $1 - \frac{\delta}{4}$ $\textsc{Na\"ive Elimination}$ returns a $(1 - \alpha)\epsilon$ approximate best arm to an $\alpha \epsilon$ approximate best arm, which gives an $\epsilon$ best arm as required. Again, a union bound shows that the probability of error is at most $\frac{\delta}{4}+\frac{\delta}{2}< \delta$.
\end{proof}

\paragraph{Sample complexity.}  The sample complexity of \textsc{Abaleh} is the sum of the sample complexity of its three procedures:
\begin{enumerate}
  \item The first iteration has sample complexity
 $$\left (1 + \frac{\lambda}{2} \right) \left ( \frac{n }{2\epsilon^2} \log \frac{1}{\delta} \right )$$
 \item The sample complexity of calling $\textsc{Aggressive Elimination}(A_0,\epsilon\alpha,\frac{\delta}{4})$ is
 $$\frac{10 \lambda n }{50\cdot \epsilon^2 \alpha^2}\log \frac{4}{\delta} < \frac{99 \lambda}{200} \left (\frac{  n }{2\epsilon^2}\log \frac{1}{\delta} \right ) $$
 where we substituted $\alpha$ and used $\lambda < 1$;
 \item Running $\textsc{Na\"ive Elimination}(A_{T}\cup R,(1\text{-}\alpha)\epsilon,\frac{\delta}{4})$ when  $n > 1/ \delta_0$
has sample complexity at most:
 $$\frac{2 n^{3/4}}{\epsilon^2} \left ( \log n + \log\frac{1}{\delta}\right ) < \frac{\lambda}{100}  \left (\frac{n}{2\epsilon^2} \log\frac{1}{\delta} \right )$$
\end{enumerate}

\section{Distributional Assumptions}\label{sec:gauss}
Throughout the paper we use the assumption the the arms' Distributions are bounded in $[0,1]$ in order to use the following version of the Hoeffding's inequality: 
\[\Pr(\hat{X}-\mathbb{E}[X]\geq t)\leq e^{-2nt^2}\]
, where $n$ is the number of samples from a given arm, $X$ is the random variable for the sum of all of the samples from this arm and $\hat{X}$ is it its realization. 
The above bound holds for any sub-Gaussian distribution with a variance $\sigma^2$ which is smaller than some constant $\sigma_0^2$. Our results may be generalized for any sub-Gaussian distribution by scaling down the values and adjusting the selection of $\epsilon$, this will effect both the upper and lower bound in the same manner and the algorithmic results are still tight.

%\input{related}

%\begin{figure*}[t]\label{fig:time}
%    \centering
% \begin{center}
%    \includegraphics[scale = .15,angle=0,origin=c]{plots/bars1}
%      \includegraphics[scale = .15,angle=0,origin=c]{plots/bars1}
%            \includegraphics[scale = .15,angle=0,origin=c]{plots/bars3}
%      \includegraphics[scale = .15,angle=0,origin=c]{plots/bars4}
%           \includegraphics[scale = .15,angle=0,origin=c]{plots/lines1}
%     \includegraphics[scale = .15,angle=0,origin=c]{plots/lines2}
%          \includegraphics[scale = .15,angle=0,origin=c]{plots/lines3}
%     \includegraphics[scale = .15,angle=0,origin=c]{plots/lines4}
%\caption{Comparison of sample complexity.  Each histogram on the top and corresponding plot on the bottom depicts a comparison between the sample complexity of \textsc{Aba} and \textsc{Median Elimination}, with a fixed choice of $(\epsilon,\delta)$ and varying values of $n$ .
%The results depict choices of $(\epsilon,\delta)$ of (from left to right) $(.1,.05),(.1,.01),(.1,.005),(.1,.001)$ as a function of the number of arms.
%In each histogram the number of samples of \textsc{Aba} (blue) and \textsc{Median Elimination} (red) is depicted on a logarithmic scale as a function of the number of arms $n^i$ for every $i \in \{1,\ldots,20\}$.  Each plot on the bottom depicts the ratio between their sample complexity as a function of $n^i$ for every $i \in \{1,\ldots,20\}$.}
%    \end{center}
%\end{figure*}

\section{Experiments}\label{sec:Experiments}
To illustrate the efficiency of the algorithms we conducted a simple numerical experiment.  A reasonable concern may be that while our results suggest a dramatic improvement over the sample complexity of \textsc{Median Elimination} this improvement may only be due to tighter analysis. In this section we rule out this possibility by experimentally comparing the actual sample complexity (not analysis) of our  algorithms (\textsc{SABA}, \textsc{ABA} and \textsc{ABALE}) with \textsc{Median Elimination} and \textsc{Na\"ive Elimination}.  Note that all algorithms are guaranteed to $(\epsilon,\delta)$-learn the best arm, and thus our interest is in their sample complexity.

Since our algorithms relative sample complexity improves as $n$ grows we were interested in observing this improvement emprically.

\paragraph{Experimental setup.}  We fixed a choice of $\delta=0.05$ and compared the sample complexity of all algorithms for $n=300,000$ arms.  Since all algorithms scale quadratically with $\epsilon$, we kept $\epsilon=0.2$ in all our experiments\footnote{We verified that changing $\epsilon$ has no effect on the ratio of the number of samples required by the algorithms.}. The arms arms are distributed in the following way: $n-1$ arms are Bernoulli random variables with mean $0.5$
and a single best arm is a Bernoulli random variable with mean $0.7+10^{-13}$.

\paragraph{Results.} We summarize the results in the table below\\  %We plot the results in Figure~\ref{fig:time}. Looking at the graphs which depict the dependence between the ratio between the sample complexity of \textsc{ABA} and \textsc{Median Elimination} and the number of arms, one can see that while \textsc{ABA} does much better for any value of $n$, there is a large improvement around $n = 10^5$. This improvement comes from moving from the $\log n$ union bound to \textsc{Aggressive Elimination}. In addition, one can see that the ratio keeps improving, until it plateaus. This has a couple of reasons. First, \textsc{Median Elimination} has a sum of a geometric sequence in its sample complexity. While this sum converges, in the beginning of the sequence it is important how many arms appear in the sum. Second, \textsc{ABA} runs \textsc{Naive Elimination} on a sublinear number of items. As $n$ grows, this stage becomes negligible. Finally, comparing the graphs for different values of $\delta$ shows that the ratio gets slightly larger when $\delta$ is small. The reason for this is that as $\delta$ becomes smaller, \textsc{ABA} can afford to move less arms to the next iteration, since we have a better guarantee for each element.

\begin{tabular}{c||c|c}
	%\hline 
	Algorithm & Average number of samples for instance & Success (out of 1000 experiments) \\ 
	\hline 
	\textsc{Median Elimination} & $9.26\cdot10^{9}$ & 1000 \\ 
	%\hline 
	\textsc{Na\"ive Elimination} & $2.34\cdot10^{8}$ & 1000 \\ 
	%\hline 
	\textsc{SABA} & $8.59\cdot10^{6}$ & 1000 \\ 
	%\hline 
	\textsc{ABA} & $1.98\cdot10^{8}$ & 1000 \\ 
	%\hline 
	\textsc{ABALE} & $8.59\cdot10^{7}$ & 1000 \\ 
	%\hline 
\end{tabular}

\textsc{SABA} is making assumptions on the input (which hold for this scenario) and is 1000 times more efficient than \textsc{Median Elimination}.
Without assumptions on the input, \textsc{ABALE} have a sample complexity which is 100 times more efficient than \textsc{Median Elimination}.
In fact, even the naive approach is more efficient than \textsc{Median Elimination}.

\section{Discussion}
The main theoretical result of this paper is an algorithm for $(\epsilon, \delta)$-\texttt{PAC} learning the best arm, with sample complexity arbitrarily close to applying the Hoeffding bound $n$ times. While the guarantees of this specific algorithm only hold for the difficult parameter regime (small $\delta$ and large $n$), simpler variants of the algorithm can be applied to any value of $n$ and reasonable choices of $\delta$.  In our experiments we compared our algorithm with \textsc{Median Elimination}, and showed a dramatic reduction in sample complexity.  Moreover, these differences grow as $n$ becomes larger.%, since a large fraction of the queries of our algorithm happen on the last $n^{3/4}$ elements. As $n$ grows larger, this become negligible, even when we use a large number of samples on each element in the final stage.

\end{document}